\documentclass{article}

 \usepackage[main, final]{neurips_2026}
\usepackage[utf8]{inputenc} % allow utf-8 input
\usepackage[T1]{fontenc}    % use 8-bit T1 fonts
\usepackage{hyperref}       % hyperlinks
\hypersetup{hidelinks}
\usepackage{url}            % simple URL typesetting
\usepackage{booktabs}       % professional-quality tables
\usepackage{amsfonts}       % blackboard math symbols
\usepackage{nicefrac}       % compact symbols for 1/2, etc.
\usepackage{microtype}      % microtypography
\usepackage{xcolor}         % colors

\usepackage{graphicx}       % figures
\usepackage{subcaption}     % subfigures with (a)/(b)/(c) labels

\usepackage{algorithm}
\usepackage{algorithmic}
\usepackage{amsmath,amsfonts,amsfonts,bm,mathtools}

\DeclareMathAlphabet{\mathsfit}{\encodingdefault}{\sfdefault}{m}{sl}
\SetMathAlphabet{\mathsfit}{bold}{\encodingdefault}{\sfdefault}{bx}{n}

\newcommand{\E}{\mathbb{E}}

\DeclareMathOperator*{\argmax}{arg\,max}

\newcommand{\order}{\mathcal{O}}

\allowdisplaybreaks % Dengwang: allow aligned equations to break into multiple pages
\mathtoolsset{showonlyrefs} % Dengwang: only number ref'ed equations

\usepackage[textsize=tiny]{todonotes}
\usepackage{amsmath,amsthm}
\usepackage[most]{tcolorbox}
\usepackage[tikz]{mdframed}
\mdfsetup{outerlinewidth=1pt,innerlinewidth=0pt, outerlinecolor=black, roundcorner=2pt}
\usepackage{wrapfig}

\newcount\Comments
\newcommand{\kibitz}[2]{\ifnum\Comments=1{\textcolor{#1}{\textsf{\footnotesize #2}}}\fi}

\newcommand{\cE}{\mathcal{E}}

\newcommand{\cN}{\mathcal{N}}
\newcommand{\cS}{\mathcal{S}}
\newcommand{\cA}{\mathcal{A}}

\newcommand{\qth}{q^{\theta}}
\renewcommand{\Pr}{\mathbb{P}}

\newcommand{\BR}{\mathfrak{BR}}

\newcommand{\ipsrl}{\texttt{inf-iPSRL}}
\newcommand{\irlsvi}{\texttt{inf-iRLSVI}}

\theoremstyle{plain}
\newtheorem{theorem}{Theorem}[section]

\newtheorem{lemma}[theorem]{Lemma}
\newtheorem{corollary}[theorem]{Corollary}
\theoremstyle{definition}

\newtheorem{assumption}[theorem]{Assumption}
\theoremstyle{remark}
\newtheorem{remark}[theorem]{Remark}

\title{Informed Posterior Sampling: More Efficient Online Learning with Few Offline Demonstrations in Average-Reward MDPs}

\author{%
  Dengwang Tang\thanks{Equal contribution.}\textsuperscript{\phantom{*},}\thanks{Work done while at the University of Southern California.}\\
  eBay Inc.\\
  \texttt{dwtang@umich.edu}
  \And
  Dongze Ye\footnotemark[1]\\
  University of Southern California\\
  \texttt{dongzeye@usc.edu}
  \And
  Botao Hao\\
  Meta TBD Lab\\
  \texttt{haobotao000@gmail.com}
  \And
  Rahul Jain\footnotemark[2]\\
  University of Southern California\\
  Google DeepMind\\
  \texttt{rahul.jain@usc.edu}
  \And
  Zheng Wen\\
  OpenAI\\
  \texttt{zhengwen@alumni.stanford.edu}
}

\usepackage[suppress]{color-edits} % Clean copy; retain edit markers in source.
\addauthor{ydz}{purple}
\begin{document}

\maketitle

\begin{abstract}
We study efficient online reinforcement learning in the infinite horizon setting when there is an offline dataset to start with. We assume that the offline dataset is generated by an expert but with an unknown level of competence, i.e., it is not perfect and not necessarily using the optimal policy. We show that if the learning agent models the behavioral policy (parameterized by a competence parameter) used by the expert, it can do substantially better in terms of minimizing cumulative regret than if it does not do that. We establish an upper bound on regret of the exact informed PSRL (iPSRL) algorithm that scales as $\Tilde{\order}(\sqrt{T})$. This requires a novel prior-dependent regret analysis of Bayesian online learning algorithms for the infinite horizon setting. We then propose the informed RLSVI (iRLSVI) algorithm to efficiently approximate the iPSRL algorithm. Empirical results show that small offline expert datasets are sufficient for iRLSVI to outperform online-only algorithms on hard exploration domains.
\end{abstract}

\section{Introduction}
\label{sec:intro}
 
The original vision of Reinforcement Learning (RL) was that of a learning agent taking actions, getting state and reward feedback and learning from it. The learning agent may have some prior information but learning mostly happens from these online interactions. This requires carefully-crafted strategies that balance exploration and exploitation, including the need for \textit{deep exploration} in many sequential decision-making problems. This was regarded as a difficult problem for a long time. Fortunately, over the last decade or so, substantial progress has been made on the \textit{online reinforcement learning} (ORL) problem. Unfortunately, ORL algorithms tend to be rather ``data-hungry'', i.e., they require a lot more interaction data than other types of learning methods. 

In many real-world problems, offline datasets are available. There is thus a possibility of improving learning efficiency by using offline datasets for pre-training, and then the learning agent does further fine-tuning upon deployment for online learning. For example, such offline data could be \textit{demonstrations from an expert, or even a sub-expert}. This allows for addressing the distribution shift problem with the offline RL methods, i.e., policies learnt offline do not perform well because the distribution of the system upon deployment is different from the one encountered offline. One can also expect that when such offline datasets are available for pre-training, it would improve online learning efficiency (in terms of cumulative regret). The question is by how much, and how does it depend on the dataset quality, and indeed how do we even quantify dataset quality. Furthermore, even with the same dataset, a question arises on how to extract the most from such a given dataset. 

An important reason why a combination of offline and online learning is important is that RL from offline datasets can often suffer from \emph{distribution shift}, i.e. the distribution of the system upon deployment is different from the one encountered offline, and hence the policy learnt offline does not perform well upon deployment. Learning from an offline dataset first followed by online learning allows for adaptation of learnt policies to different distributions of the underlying environment.

In this paper, we develop a systematic method to incorporate expert demonstrations into online learning algorithms to facilitate faster learning for infinite horizon MDPs. We first propose the idealized \emph{infinite-horizon informed Posterior Sampling based Reinforcement Learning} (\ipsrl{}) algorithm. Under some mild assumptions on the underlying MDP, we show that the algorithm achieves $\order(1) + \Tilde{\order}(\sqrt{\varepsilon T })$ regret, where $\varepsilon$ is the estimation error probability of the optimal policy given the offline data, and $T$ is the learning horizon. The regret bound also has polynomial dependence on problem parameters of the underlying MDP. We also show that, if the offline data is generated by an expert with suitably high competence level, under certain assumptions, the regret of the \ipsrl{} algorithm goes to $\order(1)$ as the size of the demonstration dataset goes to infinity.
We would also like to point out the \ipsrl{} algorithm has a simpler design compared to the TSDE algorithm in
\cite{ouyang2017learning}, since the TSDE algorithm uses a visitation-count-based episode schedule while the \ipsrl{} algorithm uses fixed episode schedules. However, the \ipsrl{} algorithm is still impractical due to the complicated nature of exact posterior updates. Therefore we introduce the \emph{informed randomized least-squares value iteration} (\irlsvi{}) algorithm which replaces exact posterior sampling with an approximation procedure. Just like the RLSVI algorithm of \cite{osband2019deep} (from which we borrowed the name), the \irlsvi{} algorithm generates approximate posterior samples via optimizing a randomly perturbed loss function.

\textbf{Related Work.} Recently, offline RL, where offline datasets are incorporated \citep{levine2020offline}, has attracted some interest due to the widespread practice of pre-training with offline datasets in large language models \citep{brown2020language,thoppilan2022lamda,hoffmann2022training}. One fundamental challenge in this line of research is the \emph{distribution shift} problem, where the policy learnt from offline data performs poorly in the real world. To address this problem, \cite{jin2021pessimism,rashidinejad2021bridging,xie2021bellman} adopted the \emph{pessimistic} approach, which can be overly conservative in practice. In \cite{uehara2021pessimistic,rashidinejad2021bridging,xie2021bellman,agarwal2022model}, the authors identified sufficient conditions on the dataset that guarantee a certain level of performance for offline RL algorithms. However, determining whether the dataset meets these conditions can be a complicated problem itself, hence making these results impractical \citep{kumar2020conservative,nair2020awac,argenson2020model,levine2020offline,kostrikov2021offline,wagenmaker2022leveraging}. There has also been some experimental work on leveraging offline data for online learning \citep{zheng2023adaptive,feng2023finetuning,hu2023imitation}.

On the other hand, there have been many works on ORL (See \cite{russo2018tutorial} for a survey). Among these works, there are two prominent approaches: the optimism in the face of uncertainty (OFU) approach \citep{auer2008near} and posterior sampling (PS) approach \citep{osband2013more,russo2016information,ouyang2017learning}. Most work in online RL focuses on the online dataset and does not consider leveraging the offline dataset. Closely related to the idea of using the offline dataset is the concept of imitation learning \citep{schaal1996learning,hester2018deep,beliaev2022imitation}, where one aims to learn the behavioral policy of the expert from the offline dataset. No online finetuning is present in these works. In \cite{ernst2005tree,vecerik2017leveraging,rashidinejad2021bridging,hansen2022modem,kumar2022should,lee2022offline}, the authors combine imitation learning with more traditional offline RL methods. In \citep{schrittwieser2021online,uehara2021pessimistic,xie2021policy,agarwal2022model,song2022hybrid,fang2022planning,wan2022safe,ball2023efficient}, the authors combined offline RL with limited online policy fine-tuning to minimize the simple regret. 

\paragraph{Main contributions.} 
(i) We introduce the informed PSRL algorithm for infinite-horizon average-reward MDPs that incorporates prior data and takes deterministic episode lengths as input (compared to \cite{ouyang2017learning} where the episode schedule is fixed and dependent on visitation counts).
(ii) We develop an optimal policy estimation-error based technique that differs from the existing information-ratio based techniques \citep{russo2016information,hao2023leveraging} for prior-dependent regret analysis. 
(iii) We provide a novel prior-dependent regret bound for learning infinite horizon MDPs that scales polynomially with all relevant parameters. 
(iv) We introduce the informed RLSVI algorithm (\irlsvi{}) as a practical approximation to \ipsrl{}. Its randomized objective combines value-iteration and imitation losses, connecting online RL with imitation learning.
(v) Empirically, we show that small offline expert datasets are sufficient for iRLSVI to outperform various baselines on hard-exploration domains.

\section{Preliminaries}
\label{sec:prelims}
\paragraph{Notations.}
We use script letters for sets. For a set $\mathcal{X}$, we use $\Delta(\mathcal{X})$ to represent the set of probability measures on $\mathcal{X}$. 
% $\Pr, \E$ stand for probability and expectation respectively. 
We use $\order$ and $\Tilde{\order}$ to hide absolute constants and logarithmic terms.

% We first formally introduce our model, assumptions and problem statement.

%\textbf{Infinite Horizon MDP}: 
\paragraph{Problem formulation.}
We consider an infinite horizon, time-homogeneous Markov Decision Process (MDP) under the average reward criterion. The MDP is represented by a tuple $\mathcal{M} = (\cS, \cA, P, r, \nu)$, where $\cS$ is the state space; $\cA$ is the action space; $P:\cS\times\cA\mapsto \Delta(\cS)$ is the transition kernel (where $\Delta(\cS)$ represents the set of distributions on $\cS$); $r:\cS\times\cA\mapsto [0, 1]$ is the instantaneous reward function; $\nu\in\Delta(\cS)$ is the distribution of the initial state. At time $t=0$, the initial state $s_0$ is drawn from $\nu$. At each time $t\in\mathbb{Z}_+$, the agent chooses an action $a_t\in\cA$, and the next state $s_{t+1}$ is chosen according to the distribution $P(\cdot|s_t, a_t)$. The objective of the agent is to choose a policy $\pi$ to maximize the expected average reward over the infinite horizon, i.e. maximizing 
\newcommand{\rbar}{\Bar{r}}
\begin{equation}
    \bar{r}^{\pi} := \lim_{T\rightarrow\infty} \dfrac{1}{T}\E^{\pi}\left[\sum_{t=0}^{T-1} r(s_t, a_t)\right].
\end{equation}
We assume that the learning agent does not possess full knowledge about the environment $\mathcal{M}$, except some prior knowledge about certain parameters of $\mathcal{M}$. To provide a general model for such prior knowledge, we use $\mathcal{M}(\theta)$ to denote an environment parameterized by some parameter $\theta$ that lies in the set $\varTheta$. The true environment is given by a random variable $\theta^\star$ with distribution $\mu_0\in\Delta(\varTheta)$. We make the following assumption on the environment.
\begin{assumption}\label{assump:comm}
    For all $\theta\in\varTheta$, $\mathcal{M}(\theta)$ is \emph{communicating}, i.e., for every pair of states $s, s'\in\cS$ there exist a stationary Markov policy $\pi$ and an integer $k\in\mathbb{N}$ with $\Pr_\theta^\pi(s_k = s' \mid s_0 = s) > 0$.
\end{assumption}

In a general infinite horizon MDP, it is possible that one decision made at a certain time can have an irreversible effect on the total reward. 
% It is possible that one enters a communicating class with suboptimal average reward. 
Intuitively, Assumption \ref{assump:comm} ensures that one could explore the MDP using any policy since there's always a pathway to return to the ``optimal part'' of the state space. See \cite{ouyang2017learning} for more discussions. 

The standard average reward MDP theory \citep{bertsekas1995dynamic} shows that for a communicating MDP, with $V_\gamma^\star(s) := \max_\pi \E^\pi[\sum_{t=0}^\infty \gamma^t r(s_t, a_t) \mid s_0 = s]$, and an arbitrary reference state $\hat s \in \cS$, we have
\begin{align}
    \bar{r}^\star &:=\lim_{\gamma\rightarrow 1} (1-\gamma) V_{\gamma}^\star(\hat{s}),\label{eq:def:J}\\
    v(s)&:=\lim_{\gamma\rightarrow 1} [V_{\gamma}^\star(s) - V_\gamma^\star(\hat{s})]\quad\forall s\in\cS.\label{eq:def:v}
\end{align}
where $\bar{r}^\star$ is the \textit{optimal expected average reward} and $v(s)$ is referred to as the \textit{relative value function}.

The pair $(\bar r^\star, v)$ satisfies the Bellman equation 
\[
\bar r^\star + v(s) = \max_{a\in\cA}\bigl[r(s,a) + \sum_{s'\in\cS} P(s'|s,a)\,v(s')\bigr],
\]
and the relative state-action value function
\begin{equation}\label{eq:def:q}
    q(s, a):=r(s, a) + \sum_{s'\in\cS}P(s'|s, a) v(s')
\end{equation}
characterizes optimality. That is, any policy of the form $\pi': \cS\to\cA$ with $\pi'(s)\in\argmax_{a\in\cA}q(s, a)$ for all $s\in\cS$ is an \textit{optimal policy} with $\bar{r}^{\pi'} = \bar{r}^\star$.

%\textbf{Learning Agent's Prior Knowledge about the Environment}: 

For a communicating MDP $\mathcal{M}(\theta)$, we use $\bar{r}^\star(\theta), v(\cdot; \theta)$, and $q(\cdot,\cdot;\theta)$ to denote their corresponding $\bar{r}^\star, v,$ and $q$ functions, as defined in \eqref{eq:def:J}, \eqref{eq:def:v} and \eqref{eq:def:q} respectively. We also make Assumption \ref{assump:span}, which is the same as that in  \cite{ouyang2017learning}.

\newcommand{\vbar}{\Bar{v}}
\begin{assumption}\label{assump:span}
    There exists $\vbar \in\mathbb{R}_+$ such that $\max_{s\in\cS} v(s; \theta) - \min_{s\in\cS} v(s;\theta) \leq \vbar$ for all $\theta\in\varTheta$. 
\end{assumption}

\begin{remark}
    Note that we can replace the ``for all $\theta\in\varTheta$'' statements in Assumptions \ref{assump:comm} and \ref{assump:span} with ``for $\mu_0$-almost all $\theta\in\varTheta$'' (a.k.a. holds for $\theta^\star$ with probability 1) without loss of generality: If certain $\theta\in\varTheta$ does not satisfy either assumption we can remove it from the parameter set and it will affect neither our algorithm nor the analysis.\footnote{Under certain mild measurability assumptions.}
\end{remark}

% Under the above assumptions we consider the canonical settings of tabular MDPs as well as linear value function generalization for non-tabular MDPs, i.e., the Q-value function is representable in a subspace spanned by a given set of features.

The learning agent has access to an \emph{offline dataset} $\mathcal{D}_0 = \{(\bar{s}_0,\bar{a}_0,\cdots,\bar{s}_{N-1}, \bar{a}_{N-1}, \bar{s}_{N})\}$ generated by another agent, called the expert, who interacts with the same environment before the learning agent started. 
The trajectory is generated from applying a randomized stationary Markov policy on the true MDP $\mathcal{M}(\theta^\star)$. The initial state $\Bar{s}_0$ is assumed to be a fixed, non-random state. We will describe the expert's policy and the learning agent's knowledge about it later in this section. 
During the learning process, the learning agent has access to the trajectory generated so far, which we call the \emph{online dataset} and denote as $\mathcal{H}_t = \{(s_0, a_0, \cdots, s_{t-1}, a_{t-1}, s_t)\}$.
At time $t$, the learning agent's information is given by $\mathcal{D}_t := \mathcal{D}_0 \cup \mathcal{H}_t$.

The learning agent utilizes a learning algorithm $\phi$ to interact with the environment, which at each time $t$ chooses a possibly randomized action based on full or partial information in the dataset $\mathcal{D}_t$. Under the environment $\mathcal{M}(\theta)$, the learning regret is defined as
\begin{equation}
    \mathfrak{R}_T(\phi; \theta) := T \bar{r}^\star(\theta) - \E_{\theta}^{\phi}\left[\sum_{t=0}^{T-1} r(s_t, a_t) \right],
\end{equation}
where $\E_{\theta}^{\phi}$ denotes the expectation under the possible trajectories generated by applying algorithm $\phi$ on the environment $\mathcal{M}(\theta)$. The Bayesian regret is defined as
\begin{equation}
    \BR_T(\phi) := \E_{\theta^\star\sim \mu_0} [\mathfrak{R}_T(\phi; \theta^\star)]
\end{equation}

We model the expert's policy in the following way. First, consider that the expert, who generated the offline dataset, has perfect knowledge about the underlying environment $\theta^\star$ but uses an approximately optimal policy in the demonstration. Then, the expert's policy $\pi^{\beta}$ is modelled as: 
\begin{equation}\label{eq:expert}
    \pi^{\beta}(a|s):= \dfrac{\exp(\beta(s) q(s, a;\theta^\star)) }{\sum_{a'\in\cA} \exp(\beta(s) q(s, a';\theta^\star)) }
\end{equation}
where $\beta(s) \geq 0$ is the \textit{state-dependent deliberateness} parameter: the expert acts uniformly at $\beta(s)=0$ and greedily as $\beta(s)\to\infty$. For fixed $q$ and $\beta(s)>0$, $\pi^\beta$ uniquely maximizes the entropy-regularized policy-improvement objective, with entropy weighted by $1/\beta(s)$, underlying the Soft Actor-Critic (SAC) \citep{geist2019regularized,haarnoja2018soft}. For human demonstrators, the same policy follows from Luce's choice rule with exponential weights or McFadden's random-utility model with independent Gumbel noise of scale $1/\beta(s)$ \citep{train2009discrete}. Thus, $\beta(s)$ measures decision precision even when the expert knows $q$ exactly.

Following \citep{hao2023leveraging}, we consider a general model wherein the expert only knows a noisy version of the optimal $q$-function. The expert uses the policy \eqref{eq:expert}, except that $q$ is replaced by $\Tilde{q}$, the expert's perceived state-action value function. The vector $\Tilde{q}$ follows a Gaussian distribution with mean $q$ and covariance $\mathbb{I} / \lambda^2$, where $\lambda > 0$ is the \emph{knowledgeability} parameter.
We refer to $(\beta,\lambda)$ together as the \textit{competence} of the expert and denote the expert's (stationary) policy as $\pi^{\beta, \lambda}$.

The learning agent knows the form of the expert's policy and the knowledgeability parameter $\lambda$ but may not know the deliberateness parameter $\beta$. For the sake of simplicity, when $\beta$ is unknown to the learning agent, we assume that $\beta(s) = \beta$ is an exponential random variable with probability density function
%an independent exponential prior for the sake of analytical simplicity, 
$f_2(x) = \lambda_2 \exp (-\lambda_2x)\bm{1}_{\{x > 0\}}$ where $\lambda_2 > 0$.

\section{The informed posterior sampling algorithm for average-reward MDPs}\label{sec:ipsrl}

We now introduce the \emph{Informed Posterior Sampling-based Reinforcement Learning} (\ipsrl) algorithm (Algorithm \ref{algo:ipsrl}) for infinite-horizon average-reward MDPs that combines the offline data with the powerful posterior sampling method \citep{osband2013more,ouyang2017learning}. We assume that the expert knows the exact $q(\cdot,\cdot;\theta^\star)$ function and follows the policy $\pi^\beta$ given by \eqref{eq:expert} with $\beta(s) = \beta > 0$ for all $s\in \cS$. We assume that the learning agent knows $\beta$ but not $\theta^\star$. Therefore, the posterior distribution of $\theta^\star$ given the offline data $\mathcal{D}_0$ is given by
\begin{equation}\label{eq:infor_prior}
     \mu_1(\theta^\star \in \cdot) \propto  \int_{\theta \in \cdot} \Pr_{\theta}^{\pi^{\beta}}(\mathcal{D}_0)  \mathrm{d} \mu_0(\theta),
\end{equation}
where
\begin{equation}\label{eq:bayesupdate}
    \Pr_{\theta}^{\pi^{\beta}}(\mathcal{D}_0) := \prod_{t=0}^{N-1} \theta(\bar{s}_{t+1} | \bar{s}_t, \bar{a}_t)\dfrac{\exp(\beta q(\bar{s}_t, \bar{a}_t;\theta)) }{\sum_{a'\in\cA} \exp(\beta q(\Bar{s}_t, a';\theta))}.
\end{equation}

In the {\ipsrl} algorithm we assume the availability of a planning oracle, called $\texttt{MDPSolve}$, to output an optimal policy for the MDP $\mathcal{M}(\theta)$ given input $\theta$: We assume that $\texttt{MDPSolve}$ is a deterministic mapping from the space of parameters $\varTheta$ to the space of deterministic stationary Markov strategies $\Pi$ with the following property: $\pi = \texttt{MDPSolve}(\theta)$ satisfies $\pi(s) \in \argmax_{a\in\cA} q(s, a;\theta)$ for all $s\in\cS$. Note that various dynamic programming algorithms such as value iteration, policy iteration, etc. can be used in  $\texttt{MDPSolve}$. A similar MDP solving oracle is also needed for \cite{ouyang2017learning}.

The {\ipsrl} algorithm divides the learning process into \textit{episodes}, \textit{with possibly different lengths}. The episode lengths $\{T_k\}_{k=1}^K$ are fixed and given as an input to the algorithm. Let $t_k := \sum_{j=1}^{k-1} T_j$ denote the starting time of episode $k$. At the beginning of episode $k$, the learning agent randomly samples a new environment $\Tilde{\theta}^k$, uses the planning oracle $\texttt{MDPSolve}$ to find an average-reward maximizing policy $\Tilde{\pi}^k$ for $\Tilde{\theta}^k$, and applies the policy $\Tilde{\pi}^k$ for the duration of episode $k$. Finally, after applying $\Tilde{\pi}^k$ for $T_k$ steps, the learning agent updates the posterior distribution of $\theta^\star$ via
\begin{equation}\label{eq:onlineupdate}
    \mu_{k+1}(\theta^\star \in \cdot) \propto \int_{\theta \in \cdot} \prod_{t=t_k }^{t_{k+1} - 1} \theta(s_{t+1}| s_t, a_t) \mathrm{d} \mu_k(\theta).
\end{equation}

\begin{algorithm}[!ht]
   \caption{inf-iPSRL}
   \label{algo:ipsrl}
\begin{algorithmic}
   \STATE \textbf{Input:} Prior $\mu_0\in\Delta(\varTheta)$; Offline data $\mathcal{D}_0$; Episode schedule $\{T_k\}_{k=1}^K$
  \STATE $t_1 \leftarrow 0$
  \STATE Compute $\mu_1$ with \eqref{eq:infor_prior}
  \FOR{$k = 1$ to $K$} 
	    \STATE Sample $\Tilde{\theta}^k \sim \mu_k$ 
	    \STATE $\Tilde{\pi}^k \leftarrow \texttt{MDPSolve}(\Tilde{\theta}^k)$ 
        \STATE $t_{k+1} \leftarrow t_k + T_k$
        \FOR{$t = t_k$ to $t_{k+1} - 1$}
            \STATE Take action $a_t = \Tilde{\pi}^k(s_t)$
        \ENDFOR
        \STATE 
	    Compute $\mu_{k+1}$ with \ydzedit{\eqref{eq:onlineupdate}}
	\ENDFOR
\end{algorithmic}
\end{algorithm}

\subsection{Prior dependent bound}\label{sec:priordependentbound}

In this section, we develop a prior-dependent regret bound for the \ipsrl{} algorithm in the tabular RL setting. 
The regret bound depends on the prior distribution $\mu_1$ through the \emph{initial mismatch probability} $\varepsilon=\Pr(\Tilde{\pi}^1 \neq \pi^\star)$, where  $\pi^\star = \texttt{MDPSolve}(\theta^\star)$ is the optimal policy of the underlying true MDP consistent with the planning oracle we used in the {\ipsrl} algorithm. Intuitively, if the offline data is more informative, then the initial posterior distribution on the optimal policy $\pi^\star$ is more concentrated, and then it's more likely that the initial policy obtained from posterior sampling is actually the optimal policy, meaning that $\varepsilon$ is small. We will formalize this at the end of this section. 

\begin{theorem}\label{thm:regbound}
    Let $K_T:=\max\{k:\sum_{l=1}^{k-1} T_l < T \}$. The Bayesian regret of \ipsrl{} satisfies
    \begin{equation}\label{eq:regbound}
        \BR_T(\phi^{\mathrm{iPSRL}}) \leq 3\vbar + 2\vbar(R_1 + R_2 + R_3),
    \end{equation}
    where
    \begin{align}
    R_1 &:= \varepsilon K_T,\\
    R_2 &:= \sqrt{\varepsilon S^2 A T \log(2SAK_T T)\left[1+\log\left(\frac{T}{SA}+1\right)\right] }, \\
    R_3 &:= \min\left\{\varepsilon T, SA\left(\max_{k\in[K_T]} T_k\right)\log_2\left(\frac{T}{SA}+1\right) \right\}.
\end{align}
\end{theorem}

Note that the right-hand side of \eqref{eq:regbound} depends on the episode schedule only through (i) total number of episodes with time $T$, and (ii) the maximum length of an episode within time $T$. Therefore, the best regret bound, in general, can be obtained by setting both quantities to be roughly $\Theta(\sqrt{T})$, which can be achieved when $T_k = \Theta(k)$. We formalize this in the following two corollaries.

\begin{corollary}
    Consider the episode schedule $T_k = k$. Then,  
    \begin{equation}
        \BR_T(\phi^{\mathrm{iPSRL}}) \leq \Tilde{\order}(\vbar S A \sqrt{T}),
    \end{equation}
    where $\Tilde{\order}$ hides the logarithmic terms in $S, A, $ and $T$. Furthermore, for fixed $\vbar, S, A, T$ we also have an $\order(1) + \order(\sqrt{\varepsilon})$ regret bound:
    \begin{equation}
        \BR_T(\phi^{\mathrm{iPSRL}}) \leq 3\vbar + 2\varepsilon \vbar T + \sqrt{\varepsilon}\vbar \cdot \Tilde{\order}(\sqrt{ S^2 A T}).
    \end{equation}
\end{corollary}

\begin{remark}
    The only difference between the episode schedule of the TSDE algorithm in \cite{ouyang2017learning} and {\ipsrl} algorithm with $T_k = k$ is a stopping criterion based on visitation count of all state-action pairs. Our result here shows that the Bayesian regret is still $\Tilde{\order}(\sqrt{T})$ if one removes the visitation count based stopping criteria in the TSDE algorithm.
\end{remark}

When $\varepsilon$ is small, one can obtain a better regret bound for large $T$ by designing episode lengths accordingly.
\begin{corollary}\label{cor:linepeps}
    Let $\hat{\varepsilon} \geq \varepsilon$. Consider the episode schedule $T_k = \lceil \hat{\varepsilon}k \rceil$. Then,  %\textcolor{red}{varepsilon here?} Dengwang: fixed!
    \begin{equation}
        \BR_T(\phi^{\mathrm{iPSRL}}) \leq 3\vbar + \Tilde{\order}(\vbar S A \sqrt{\hat{\varepsilon} T}).
    \end{equation}
\end{corollary}

To better interpret the dependence of the regret on the prior distribution, we note that the initial mismatch probability can be bounded through the estimation error of the optimal policy given the offline data.
\begin{lemma}\label{lem:estimator2p1}
    Let $\hat{\pi}^\star$ be any estimator of $\pi^\star$ constructed from $\mathcal{D}_0$. Then $\Pr(\Tilde{\pi}^1 \neq \pi^\star)\leq 2 \Pr(\hat{\pi}^\star \neq \pi^\star)$.
\end{lemma}

% Moved to appendix
% \begin{proof}
%     If $\Tilde{\pi}^1 \neq \pi^\star$, then either $\hat{\pi}^\star \neq \Tilde{\pi}^1$ or $\hat{\pi}^\star \neq \pi^\star$ must be true. Conditioning on $\mathcal{D}_0$, $\Tilde{\pi}^1$ is identically distributed as $\pi^\star$ while $\hat{\pi}^\star$ is deterministic, therefore % \textcolor{red}{typo?} Dengwang: fixed!
%     \begin{align*}
%         \Pr(\Tilde{\pi}^1 \neq \pi^\star)\leq \Pr(\hat{\pi}^\star \neq \Tilde{\pi}^1) + \Pr(\hat{\pi}^\star \neq \pi^\star) =  2 \Pr(\hat{\pi}^\star \neq \pi^\star)
%     \end{align*}
% \end{proof}

\begin{remark}
    When $\varepsilon = \Theta(1)$, the regret bound in Corollary \ref{cor:linepeps} has an additional factor of $\sqrt{SA}$ compared to the $\Tilde{\order}(\bar{v}\sqrt{SAT})$ regret bound for (uninformed) learning in communicating MDPs given by \cite{agrawal2017optimistic}. However, the advantage of our prior-dependent bound is for small $\varepsilon$. The next section shows that under certain assumptions, $\varepsilon$ can be exponentially small in the amount of prior data. 
    % We leave the improvement of the prior-dependent upper bound's dependency on $S$ and $A$ as future work.
\end{remark}

\subsection{Bounding the estimation error of the optimal policy}\label{sec:bounding}

We now provide a bound on the estimation error of the optimal policy given the offline data.
In addition to Assumptions \ref{assump:comm} and \ref{assump:span}, we make the following additional assumptions.
\begin{assumption}\label{assump:gap}
    There exists a number $\Delta > 0$ such that for all $\theta\in\varTheta$ and all $s\in\cS$, there exists $a^\star\in \cA$ such that $q(s, a^\star; \theta) - q(s, a; \theta) \geq \Delta$ for all $a\neq a^\star$.
\end{assumption}

\begin{assumption}\label{assump:minprob}
    There exists a number $\delta > 0$ such that for all $\theta\in\varTheta$ and all $s, s'\in\cS$, there exists a stationary Markov policy $\pi$ such that
    \begin{equation}
        \Pr_{\theta, \pi}(\exists 0<t\leq S , s_t = s'|s_0=s  ) \geq \delta.
    \end{equation}
\end{assumption}

Here,  $S$ is the number of states. Assumption \ref{assump:minprob} can be interpreted as requiring all MDPs to be ``uniformly communicating'': If we allow $\delta$ to be dependent on $\theta$ in Assumption \ref{assump:minprob}, then Assumption \ref{assump:minprob} is the same as Assumption \ref{assump:comm}, i.e. all $\mathcal{M}(\theta)$ are communicating. 

\begin{lemma}\label{lem:piestimator}
    For a sufficiently large $\beta(s) = \beta$ and any $N\in\mathbb{N}$, there exists an estimator $\hat{\pi}^\star$ such that
    \begin{equation}
        \Pr(\hat{\pi}^\star \neq \pi^\star) \leq \varepsilon_N:=S(1 + 1.06 N)\exp(-cN)
    \end{equation}
    where $c > 0$ is a constant dependent on $(S, A, \vbar, \Delta, \delta, \beta)$.
\end{lemma}

\section{An approximate Bayesian algorithm}

\subsection{The informed RLSVI algorithm for average-reward MDPs}\label{sec:iRLSVI}

% In the previous section, we considered the \ipsrl{} algorithm where the learning agent performs an posterior update at each episode. However, computing exact posterior updates can be challenging due to the loss of conjugacy when applying Bayes rule in \eqref{eq:infor_prior}. 
In this section, we propose a novel Bayesian bootstrapping-based approach for a tractable approximation of posterior updates, where we use the maximum a posteriori (MAP) estimator based on a perturbed loss function as a surrogate for samples from exact posterior distributions.

% In this section, we consider a different setting as the previous section: 
In what follows, we consider the more general setting where 
the expert deliberateness parameter $\beta$ is unknown to the agent. Instead, $\beta(s) = \beta$ is modelled via an exponential prior distribution with density function $f_2$. At each time in the algorithm, the agent forms a joint belief over $(\theta, \beta)$ based on all the offline and online data, with prior density function $f\times f_2$.
% . The prior distribution of $(\theta, \beta)$ is given by the density function $f\times f_2$.

To unify notations for offline and online data, define $\check{s}_k$ to be $\Bar{s}_k$ if $k< N$ and $s_{k-N}$ otherwise. Define $\check{a}_k$ similarly. To simplify notations, write $q(s,a; \theta)$, defined in \eqref{eq:def:q}, as $\qth(s,a)$. Write $r_k = r(\check{s}_k, \check{a}_k)$.

Consider the average-reward Bellman optimality equation 
\begin{equation}
\bar{r}^\star + q(s,a; \theta) = r(s,a) + \mathbb{E}_{s'\sim \theta(\cdot|s, a)}\left[\max_{a'\in\cA} q(s',a'; \theta)\right].
\end{equation}

Let $\Tilde{r}_k$ be defined recursively via
$$ \Tilde{r}_{k+1} := (1-\eta_k)\Tilde{r}_k + \eta_kr_k, ~~\Tilde{r}_0:= r_0.$$

Define the \textit{temporal difference (TD) error} ${\cE}_k$ (parameterized by a given $\qth$),
$$\cE_k(\qth) : = \left(r_k + \max_{a'\in\cA} \qth(\check{s}_{k}',a') - \qth(\check{s}_k,\check{a}_k) - \Tilde{r}_k\right),$$
where $\check{s}_{k}'= \check{s}_{k+1}$ for all $k$ except when $k = N - 1$, in which case $\check{s}_k' = \Bar{s}_N$. Here $\Tilde{r}_k$ is used as a bootstrap approximation for the true optimal average reward $\Bar{r}^\star$. Therefore, conditioning on $(\theta, \check{s}_k, \check{a}_k)$, $\cE_k(\qth)$ is approximately zero-mean. Specifically, we model each TD error as $\cE_k(\qth)\mid(\theta,\check s_k,\check a_k)\sim\cN(0,\sigma^2)$.

Combined with a Gaussian prior $\theta\sim\cN(0,\Sigma_0)$ (for positive-definite $\Sigma_0$) and the exponential prior on $\beta$, the MAP objective at time $t$ decomposes into three terms: (1) a value-iteration loss, (2) an imitation loss, and (3) Gaussian and exponential log-prior regularizers (see Appendix~\ref{app:loss-derivation}). We then apply weighted-likelihood Bayesian bootstrap \citep{newton1994approximate, osband2019deep}: each TD error is perturbed by additive noise $z_k\sim\cN(0,\sigma^2)$, each offline imitation term is reweighted by an independent $w_k\sim\mathrm{Exp}(1)$, and the prior is recentered on a draw $\tilde\theta\sim\cN(0,\Sigma_0)$. The resulting randomized loss function is

\begin{equation}
\begin{split}
 \tilde{\mathcal{L}}(\theta,\beta) &=
 \underbrace{\frac{1}{2\sigma^2}\sum_{k=0}^{N+t-1}\left(r_k + z_k + \max_{a'\in\cA} \qth(\check{s}_k',a') - \qth(\check{s}_k,\check{a}_k) - \Tilde{r}_k\right)^2}_{\text{Randomized Value Iteration Loss}}\\
& \underbrace{- \sum_{k=0}^{N-1} w_k\left(\beta \qth(\Bar{s}_k, \Bar{a}_k)-\log\sum_{a'\in\cA}\exp\left(\beta \qth(\Bar{s}_k, a')\right)\right)}_{\text{Randomized Imitation Loss}}+ \underbrace{\frac{1}{2} (\theta - \tilde{\theta})^\top\Sigma_0^{-1}(\theta - \tilde{\theta}) +\lambda_2\beta}_{\text{Randomized Prior Loss}} .
\end{split}
\label{eq:Ltilde}
\end{equation}

This yields the informed RLSVI algorithm (\irlsvi{}) for infinite-horizon average-reward MDPs, which approximates \ipsrl{} by, at each episode, drawing perturbations and solving the randomized loss  function \eqref{eq:Ltilde}. The pseudocode for \irlsvi{} is given in Appendix~\ref{app:loss-derivation}.

% \begin{remark}
%     The MAP estimation of $\beta$ can be quite noisy when $\beta$ is large due to the fact that the expert takes non-optimal actions with very small probability. To resolve this, we may consider an entropy-based estimate of $\beta$: The estimate is set to equal $c_0/ \hat{H}$, where $\hat{H}$ is the conditional entropy of $\hat{a}$ given $\hat{s}$ if $(\hat{s}, \hat{a})$ follows the empirical distribution given by visitation counts in the offline data.
% \end{remark}

\begin{remark}
    Due to the maximum operation inside the TD target, $\max_{a'\in\cA} \qth(\check{s}_k',a')$, and the nonlinearity of the $q$-function with respect to $\theta$, optimizing \eqref{eq:Ltilde} exactly can be difficult in general.
    A simple relaxation, analogous to the target network in DQN \citep{mnih2015human}, is to compute the TD target using a frozen parameter $\theta^k$ from the previous iteration and optimize the remaining terms. Our experiments in Section~\ref{sec:experiments} confirm that this simplification is sufficient for strong empirical results.
\end{remark}

\subsection{Informed RLSVI bridges online RL and imitation learning}

In this section, we present a detailed interpretation of the randomized loss function to show that \irlsvi{} bridges online RL algorithms \citep{osband2019deep} with imitation learning.

\citet{osband2019deep} consider episodic learning of finite horizon MDPs with horizon $H$. The algorithm was inspired by the posterior sampling approach and has great empirical performance. At the end of episode $t$, the RLSVI algorithm considers a randomized loss function
\begin{equation}\label{eq:loss-rlsvi}
\begin{split}
\tilde{\mathcal{L}}_{\text{RLSVI}}(\theta) 
&=  \frac{1}{2\sigma^2}\sum_{k=1}^{N+t} \sum_{h=0}^{H-1}\left(r_{k, h} + z_{k, h} + \max_{a'\in\cA} Q^{\theta}_{h+1}(\check{s}_{k, h+1},a') - Q^{\theta}_h(\check{s}_{k, h},\check{a}_{k,h})\right)^2 
\\&+ \frac{1}{2} ({\theta_{0:H}} - {\tilde{\theta}_{0:H}})^\top\Sigma_0^{-1}({\theta_{0:H}} - {\tilde{\theta}_{0:H}}). \end{split}
\end{equation}
where $z_{k, h}\sim \mathcal{N}(0, \sigma^2)$ are random perturbations.
% The sum of randomized value iteration loss and randomized prior loss in \eqref{eq:Ltilde} can be seen as an adaptation of \eqref{eq:loss-rlsvi} for infinite horizon MDP learning problems.

Let $\hat{\pi}(a|s)$ denote the empirical distribution of actions at state $s$ in the offline dataset $\mathcal D_0$. The KL-divergence imitation learning loss at state $s\in\cS$ can be specified by
\begin{align}
\mathcal{L}_{\text{IL}}(\beta,\theta, s) &= 
D_{KL}\left(\hat{\pi}(\cdot|s)\|\pi^{\beta}(\cdot|s; q^{\theta})\right)  = \sum_{a}\hat{\pi}(a|s) \log\left(\dfrac{\hat{\pi}(a|s)}{\pi^{\beta}(a|s; q^{\theta}) }\right) \\
&=  \sum_{a}\hat{\pi}(a|s) \log\left(\dfrac{1}{\pi^{\beta}(a|s; q^{\theta}) }\right) + \text{constant} \\
&= -\dfrac{1}{\Bar{N}(s)} \sum_{k=0}^{N-1}\sum_{a\in\cA} \bm{1}_{\{\Bar{s}_k = s, \Bar{a}_k = a \}} \log\left(\pi^{\beta}(a|s; q^{\theta}) \right) + \text{constant}\\
&= -\dfrac{1}{\Bar{N}(s)} \sum_{k=0}^{N-1} \bm{1}_{\{\Bar{s}_k = s\}} \left(\beta \qth(\Bar{s}_k, \Bar{a}_k)-
\log\sum_{a'\in\cA}\exp\left(\beta \qth(\Bar{s}_k, a')\right)\right) 
 + \text{constant}.
\end{align}%\label{eq:loss-bc}
where $\Bar{N}(s)$ is the visitation count of state $s$ in the offline data.

Therefore, the (pre-randomization) imitation loss in \eqref{eq:Lthetabeta} can be written as a weighted sum of the individual imitation losses defined above
\begin{align}
    \text{Imitation Loss in \eqref{eq:Lthetabeta}} = \sum_{s\in\cS} \Bar{N}(s) \mathcal{L}_{\text{IL}}(\beta,\theta, s) - \text{constant}
\end{align}

Therefore, we have shown that the randomized loss \eqref{eq:Ltilde} effectively bridges online RL with imitation learning by combining the RLSVI loss for online learning and the randomized imitation loss. 

% Empirical Section

\section{Experiments}
\label{sec:experiments}

\paragraph{Setup and baselines.}
We implement \irlsvi{} with the standard \textit{frozen-target trick} (Algorithm~\ref{app-alg:inf_irlsvi}, Appendix~\ref{app:practical-alg}) with a fixed episode length $T_k = 100$.\footnote{Code is available at \url{https://github.com/dongzeye/informed-psrl}} We compare with online-only baselines: IMED-RL \citep{pesquerel2022imed}, UCRL3 \citep{bourel2020tightening}, and uninformed RLSVI (uRLSVI).\footnote{UCRL3 and IMED-RL implementations are adapted from \citep{pesquerel2022imed} under the MIT license.} To isolate the contribution of expert-competence modelling, we also incorporate \emph{partially informed} baselines (``pi'' prefix) that receive the same offline dataset but treat it as ordinary transition data.\footnote{Principled offline-to-online extensions of IMED-RL or UCRL3 would require non-trivial modifications, which we leave to future work.} Each informed agent receives offline data of ratio $\kappa = N/(SA) = 5$ from expert policy~\eqref{eq:expert} (per-environment $\beta$ in Table~\ref{tab:main_results}). Hyperparameters and further details are in Appendix~\ref{app:experiments}.

\paragraph{Environments.}
We use \textit{RiverSwim}-$S$ ($S$ states) \citep{pesquerel2022imed}, a chain MDP with rewards only at the endpoints: a small reward is readily accessible at one end, while reaching the larger reward at the other requires sustained exploration. We also consider two gridworlds, \textit{Grid4Room} and \textit{Grid2Room} \citep{bourel2020tightening}, where the agent must navigate interconnected rooms under stochastic transitions to reach a rewarding state. Finally, we introduce \textit{ClusterRing}-$M$, a sparse-reward domain formed by $M$ random MDPs arranged in a directed cycle whose optimal policy traverses the entire ring. Full specifications are in Appendix~\ref{app:envs}.

\begin{table}[ht!]
\footnotesize
\setlength{\tabcolsep}{4pt}
\centering
\caption{Cumulative regret at $T=100k$, averaged over 20 seeds. Informed methods use $N = 5SA$ offline data; $\beta$ is expert deliberateness. \textbf{Bold:} best (lowest) regret per environment. \textbf{Dash (--)}: computationally intractable.}
\label{tab:main_results}
\begin{tabular}{lrrrrrrrrrr}
\toprule
Environment & $S$ & $A$ & $\beta$
& uRLSVI & piRLSVI & iRLSVI & IMED-RL & piIMED-RL & UCRL3 & piUCRL3 \\
\midrule
RiverSwim-6      &   6 &  2 & 3 & 11.7k & 2.2k & 180 & 80 & \textbf{59} & 565 & 549 \\
RiverSwim-25     &  25 &  2 & 3 & 38.6k & 33.5k & \textbf{108} & 24.8k & 314 & 4.6k & 4.6k \\
RiverSwim-50     &  50 &  2 & 3 & 38.6k & 38.7k & \textbf{137} & 37.9k & 1.1k & 12.4k & 11.0k \\
RiverSwim-100    & 100 &  2 & 3 & 38.5k & 38.6k & \textbf{202} & 37.9k & 4.6k & 32.6k & 33.2k \\
\midrule
ClusterRing-5   &  50 & 10 & 50 & 10.4k & 3.7k & \textbf{168} & 5.4k & 1.4k & 11.7k & 12.0k \\
ClusterRing-10  & 100 & 10 & 50 & 5.0k & 3.0k & \textbf{213} & 3.5k & 2.0k & 5.0k & 5.1k \\
ClusterRing-20  & 200 & 10 & 50 & 2.0k & 1.3k & \textbf{599} & 1.7k & 1.4k & -- & -- \\
\midrule
Grid4Room       &  20 &  4 & 50 & 6.3k & 3.2k & 161 & 72 & \textbf{70} & 4.6k & 4.6k \\
Grid2Room       &  55 &  4 & 50 & 4.7k & 3.4k & 760 & 145 & \textbf{138} & 4.3k & 4.4k \\
\bottomrule
\end{tabular}
\end{table}

\begin{figure}[t]
\centering
\includegraphics[width=\linewidth]{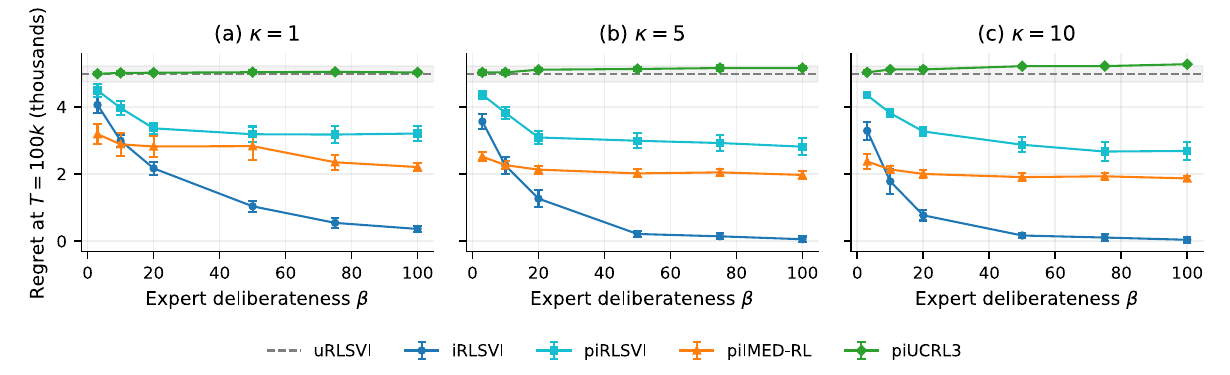}
\caption{Cumulative regret at $T=100k$ vs.\ expert deliberateness $\beta$ on \textit{ClusterRing}-10 across data ratios (a) $\kappa=1$, (b) $\kappa=5$, (c) $\kappa=10$. Mean over 20 seeds; bars show 95\% confidence intervals.}
\label{fig:regret_vs_beta_main}
\end{figure}

\paragraph{Main results.}
On hard-exploration domains (larger \textit{RiverSwim}s and all \textit{ClusterRing}s), \irlsvi{} attains the lowest regret among the baselines in Table~\ref{tab:main_results}, often improving on the strongest online-only competitor by one to two orders of magnitude.
% Its \textit{RiverSwim} regret stays below $250$ for $S \le 100$ while every online-only baseline saturates above $30k$ -- 
This provides direct evidence that offline demonstrations enable deep exploration where online algorithms alone fail. The piRLSVI-to-iRLSVI gap (e.g., $33.5k$ to $108$ on \textit{RiverSwim-25}) reveals the benefits of modelling expert competence, as opposed to treating offline data as ordinary transitions. On \textit{RiverSwim}-6 and the gridworlds, IMED-RL variants have the lowest regret, but \irlsvi{} remains competitive and outperforms its ablations by a large margin.

\paragraph{Effect of demonstration quality ($\beta$) and quantity ($\kappa$).}
Figure~\ref{fig:regret_vs_beta_main} shows \irlsvi{} can extract the expert-competence signal even from limited offline data: regret falls almost monotonically with $\beta$ at every $\kappa \in \{1,5,10\}$, and larger $\kappa$ further reduces regret. As $\beta$ increases, the rate of improvement for partially informed baselines diminishes rapidly, suggesting that they are less effective at learning from high-quality demonstrations. These trends hold across different environments (as shown in Appendix~\ref{app:beta_grid}).

\begin{table}[t]
\centering
\caption{Cumulative regret under expert policy misspecification at $T=100k$, averaged over 20 seeds ($\kappa=5$). Each pair compares iRLSVI with its partially informed ablation using the same demonstrations (see details in Appendix~\ref{app:robustness}).}
\label{tab:robustness-main}
\setlength{\tabcolsep}{3pt}
\begin{tabular}{lrrrrrr}
\toprule
 & \multicolumn{2}{c}{RiverSwim-50} & \multicolumn{2}{c}{ClusterRing-10} & \multicolumn{2}{c}{Grid4Room} \\
\cmidrule(lr){2-3}\cmidrule(lr){4-5}\cmidrule(lr){6-7}
Expert policy & iRLSVI & piRLSVI & iRLSVI & piRLSVI & iRLSVI & piRLSVI \\
\midrule
Boltzmann, $q^\star$, $\beta(s)$ & 5,978 & 38,544 & 925 & 3,278 & 655 & 3,490 \\
Boltzmann, $\tilde q$, scalar $\beta$ & 2,338 & 38,745 & 588 & 2,871 & 1,713 & 3,476 \\
Boltzmann, $\tilde q$, $\beta(s)$ & 8,993 & 38,605 & 1,219 & 3,122 & 2,946 & 3,782 \\
$\epsilon$-greedy, $\tilde q$ & 1,793 & 38,658 & 787 & 3,881 & 1,480 & 3,880 \\
Balanced mixture & 5,051 & 39,335 & 638 & 3,517 & 1,252 & 3,733 \\
Pooled mixture & 38,532 & 38,538 & 619 & 3,224 & 1,299 & 3,725 \\
\bottomrule
\end{tabular}
\end{table}

\paragraph{Robustness under expert policy misspecification.}
We test whether practical \irlsvi{} remains effective under biased values, state-dependent deliberateness, $\epsilon$-greedy behavior, and heterogeneous experts, while still fitting a single scalar $\beta$. Across four individual policies and a balanced mixture in three environments, \irlsvi{} outperforms uRLSVI and its direct piRLSVI ablation in all 15 comparisons. Table~\ref{tab:robustness-main} reports the comparison with piRLSVI. These results suggest that the scalar-$\beta$ expert model provides a useful inductive bias even beyond the assumptions of our analysis (see details in Appendix~\ref{app:robustness}).

\paragraph{Comparison with prior methods for discounted MDPs.}
We further compare with established offline-to-online methods adapted from the discounted setting, namely RLPD \citep{ball2023efficient} and CQL-FT \citep{kumar2020conservative}. Given the same offline demonstrations, \irlsvi{} achieves lower regret than both adapted baselines in five tested environments. The full results and adaptation details are provided in Appendix~\ref{app:o2o}.

\section{Conclusions}
\label{sec:conclusions}

In this paper, we have introduced an ideal Bayesian offline-to-online RL algorithm, \ipsrl{}, that naturally uses any offline data available for pre-training to boost online learning. The algorithm is agnostic to the size of the offline data or the time window for online learning. 
% It is also useful in the special case of no offline data but not for the `pure' offline RL setting when there is no online learning phase. 
A key aspect of the proposed algorithm is that unlike past PS-based algorithms for infinite horizon MDPs, our algorithm takes the episode length as an input. Our analysis then shows what type of episodic schedules result in near-optimal regret. The analysis involves several novel elements and depends on prior-dependent regret analysis, which is known to be quite challenging and itself a novelty. 

We then presented an approximate Bayesian online learning algorithm, \irlsvi{}, that bridges online learning with imitation learning. 
% In prior work \citep{osband2019deep}, such algorithms are well-known to be computationally practical in problems that need deep exploration. 
We implemented \irlsvi{} with the standard frozen-target trick and observed strong empirical gains on hard-exploration benchmarks using even small offline datasets, with the contribution of imitation learning  growing with expert deliberateness.

We note that while it should be possible to design non-Bayesian algorithms for the problem we address in this paper, and analyze its frequentist regret, it is much less natural to incorporate information derived from offline datasets in such algorithms, and hence our choice to focus on Bayesian algorithms with a guarantee on their Bayesian regret.

\paragraph{Limitations and future work.}
Our main theoretical and empirical results are limited to tabular MDPs.
A theoretical analysis of \irlsvi{} remains open and may build on existing results for RLSVI \citep{osband2019deep} and the prior-dependent techniques developed here.
We believe that the loss function for \irlsvi{} can be directly applied to deep RL methods with appropriate tuning.
More broadly, our infinite-horizon RL approach may be extended to continual learning and adaptation of LLM agents for personalization tasks \citep{li2025personalized,qin2026memory,liang2026learning}.

\begin{ack}
This project was partially supported by the following grants: NSF Award no. 2448885, ONR MURI N00014-24-1-2742, NSF award no. 2025732.

The authors declare no competing interests.
\end{ack}

\newpage
\bibliographystyle{plainnat}
\bibliography{main}

%%%%%%%%%%%%%%%%%%%%%%%%%%%%%%%%%%%%%%%%%%%%%%%%%%%%%%%%%%%%

\clearpage
\appendix
\counterwithin{figure}{section}
\counterwithin{table}{section}
\section{MAP loss decomposition and the \irlsvi{} algorithm}\label{app:loss-derivation}

This section provides the explicit MAP loss $\mathcal{L}(\theta,\beta)$ summarized in Section~\ref{sec:iRLSVI}.
Treating each $\mathbf{p}_k=(\check{s}_k,\check{a}_k,\cE_k(\qth))$ as a separate data point with $\check{s}_k$ as a given context and modelling
$\cE_k(\qth)\mid(\theta,\check s_k,\check a_k)\sim\cN(0,\sigma^2)$, the offline actions carry additional
information (relative to the online ones) since they were generated from $\pi^\beta(\cdot\mid\cdot;\,q^\theta)$.
Then the MAP estimate of $(\theta,\beta)$ at time $t$ is given by $\arg\min_{\theta,\beta}\mathcal{L}(\theta,\beta)$ where
\begin{align}
    \mathcal{L}(\theta,\beta) &= -\sum_{k=0}^{N+t-1}\log \Tilde{\Pr}_{\theta}(\cE_k(\qth) | \check{s}_k,\check{a}_k ) - \sum_{k=0}^{N-1}\log\pi^\beta (\check{a}_k|\check{s}_k; q^{\theta}) - \log f(\theta) - \log f_2(\beta) \\
 &= \underbrace{ \frac{1}{2\sigma^2}\sum_{k=0}^{N+t-1}\left(r_k + \max_{a'\in\cA} \ydzedit{\qth(\check{s}_k',a')} - \qth(\check{s}_k,\check{a}_k) - \Tilde{r}_k\right)^2 }_{\text{Value Iteration Loss}}  \\
& \underbrace{- \sum_{k=0}^{N-1}\left(\beta \qth(\Bar{s}_k, \Bar{a}_k)-\log\sum_{a'\in\cA}\exp\left(\beta \qth(\Bar{s}_k, a')\right)\right)}_{\text{Imitation Loss}}+ \underbrace{\frac{1}{2} \theta^\top\Sigma_0^{-1} \theta  +\lambda_2\beta}_{\text{Prior Loss}} + ~\text{constant} ,\label{eq:Lthetabeta}
\end{align}
with $\Sigma_0\succ 0$ the prior covariance of the Gaussian prior on $\theta$. Drawing independent perturbations $z_k \sim \cN(0,\sigma^2)$ for each TD-error term $k=0,\ldots,N+t-1$, $w_k \sim \mathrm{Exp}(1)$ for each offline imitation term $k=0,\ldots,N-1$, and $\tilde\theta \sim \cN(0,\Sigma_0)$, and adding them to the value-iteration, imitation, and prior terms respectively, yields the randomized loss \eqref{eq:Ltilde}, restated here for completeness:
\begin{equation*}
\begin{split}
 \tilde{\mathcal{L}}(\theta,\beta) &=
 \underbrace{\frac{1}{2\sigma^2}\sum_{k=0}^{N+t-1}\left(r_k + z_k + \max_{a'\in\cA} \ydzedit{\qth(\check{s}_k',a')} - \qth(\check{s}_k,\check{a}_k) - \Tilde{r}_k\right)^2}_{\text{Randomized Value Iteration Loss}}\\
& \underbrace{- \sum_{k=0}^{N-1} w_k\left(\beta \qth(\Bar{s}_k, \Bar{a}_k)-\log\sum_{a'\in\cA}\exp\left(\beta \qth(\Bar{s}_k, a')\right)\right)}_{\text{Randomized Imitation Loss}}+ \underbrace{\frac{1}{2} (\theta - \tilde{\theta})^\top\Sigma_0^{-1}(\theta - \tilde{\theta}) +\lambda_2\beta}_{\text{Randomized Prior Loss}} .
\end{split}
\end{equation*}

The resulting \irlsvi{} algorithm (Algorithm~\ref{algo:irlsvi}) mirrors the inf-iPSRL (Algorithm~\ref{algo:ipsrl}) but replaces the exact posterior-sampling step with the MAP estimator of the randomized loss \eqref{eq:Ltilde}. A further simplification, applying the \textit{frozen-target trick}, is used for experiments and shown in Algorithm~\ref{app-alg:inf_irlsvi} (Appendix~\ref{app:practical-alg}).

\begin{algorithm}[!ht]
   \caption{inf-iRLSVI}
   \label{algo:irlsvi}
\begin{algorithmic}
   \STATE \textbf{Input:} Prior $\mu_0 \in \Delta(\varTheta)$; Offline data $\mathcal{D}_0$; Episode schedule $\{T_k\}_{k=1}^K$
   \STATE $t_1 \leftarrow 0$
   \FOR{$k = 1$ to $K$}
       \STATE Sample perturbations $\{z_k\}, \{w_k\}, \tilde\theta$ as in Section~\ref{sec:iRLSVI} and construct $\tilde{\mathcal{L}}(\theta, \beta)$ as in \eqref{eq:Ltilde}
       \STATE Solve $(\Tilde{\theta}^k, \Tilde{\beta}^k) = \arg\min_{\theta, \beta} \tilde{\mathcal{L}}(\theta, \beta)$
       \STATE $\Tilde{\pi}^k \leftarrow \texttt{MDPSolve}(\Tilde{\theta}^k)$
       \STATE $t_{k+1} \leftarrow t_k + T_k$
       \FOR{$t = t_k$ to $t_{k+1} - 1$}
           \STATE Take action $a_t = \Tilde{\pi}^k(s_t)$
       \ENDFOR
   \ENDFOR
\end{algorithmic}
\end{algorithm}

\section{Proof of Theorem \ref{thm:regbound}}
\label{app:proof-regbound}
Recall that $t_k = \sum_{l=1}^{k-1} T_l$ is the starting time of learning episode $k$.
Let $\Tilde{\mathcal{D}}_{k} := \mathcal{D}_0 \cup \{(s_0, a_0, \cdots, s_{t_k-1}, a_{t_k-1}, s_{t_k})\}$ denote the offline and online data the learning agent used to form a posterior to sample the environment for the $k$-th learning episode. For simplicity of notation, let $\Pr_k(\cdot) = \Pr(\cdot|\Tilde{\mathcal{D}}_{k})$ and $\E_k[\cdot] = \E[\cdot|\Tilde{\mathcal{D}}_{k}]$. Also define $\Tilde{\E}_k^\star[\cdot] = \E[\cdot|\Tilde{\mathcal{D}}_{k}, \theta^\star, \Tilde{\theta}^k]$. 

For convenience, we fix $T$ and write $K = K_T$. For the sake of this proof, if the last episode goes beyond time $T-1$, we ``cut off'' the episode after time $T-1$ and define $T_K = T - t_K$ and $t_{K+1} = T$.

Without loss of generality, by choosing an appropriate reference state $\hat{s}$ in \eqref{eq:def:v} we can assume that $v(s;\theta)\in [0, \vbar]$.

We will start with the following lemma which bounds all the mismatch probabilities by the initial mismatch probability.

\begin{lemma}\label{lem:psmonotone}
    $\Pr(\Tilde{\pi}^k \neq \pi^\star) \leq \Pr(\Tilde{\pi}^1 \neq \pi^\star )$ for all $k\geq 1$.
\end{lemma}

\begin{proof}
    Define $f(x) = x(1-x)$. We have
    \begin{align}
        \Pr(\Tilde{\pi}^k \neq \pi^\star ) &= \E[\Pr_k(\Tilde{\pi}^k \neq \pi^\star) ]= \E\left[ \sum_{\pi\in\varPi} \Pr_k(\Tilde{\pi}^k = \pi, \pi^\star \neq \pi)  \right]\\
        &= \E\left[ \sum_{\pi\in\varPi} f(\Pr_k(\pi^\star = \pi))  \right] \\
        &=  \E\left[\sum_{\pi\in\varPi} \E_1[f(\Pr_k(\pi^\star = \pi))]\right]\leq \E\left[ \sum_{\pi\in\varPi} f(\E_1[\Pr_k(\pi^\star = \pi)])\right]\\
        &= \E\left[\sum_{\pi\in\varPi} f(\Pr_1(\pi^\star = \pi))\right]= \E\left[ \sum_{\pi\in\varPi} \Pr_1(\Tilde{\pi}^1 = \pi, \pi^\star \neq \pi)  \right]\\
        &= \Pr(\Tilde{\pi}^1 \neq \pi^\star)
    \end{align}
    where the inequality is due to Jensen's Inequality and the fact that $f$ is a concave function. 
\end{proof}

Define the episode $k$ regret $Z_k$ as
\begin{equation}
    Z_k = T_k \bar{r}^\star(\theta^\star) - \sum_{t=t_k}^{t_{k+1}-1} r(s_t, a_t).
\end{equation}
Then we can write $\BR_T(\phi^{\mathrm{iPSRL}}) = \sum_{k=1}^K\E[Z_k]$.

If the event $\{\Tilde{\pi}^k = \pi^\star\}$ is true, then for $t_k\leq t < t_{k+1}$, we have $a_t = \pi^\star(s_t)$ and $v(s_t; \theta^\star) + \bar{r}^\star(\theta^\star) = r(s_t, a_t) + \sum_{s'\in\cS}\theta^\star(s'|s_t, a_t) v(s';\theta^\star)$ due to the assumption on the planning oracle $\texttt{MDPSolve}$. Therefore,
\begin{equation}
    Z_k \bm{1}_{\{ \Tilde{\pi}^k = \pi^\star\}} = \sum_{t=t_k}^{t_{k+1}-1} \left(\sum_{s'\in\cS}\theta^\star(s'|s_t, a_t) v(s';\theta^\star) - v(s_t;\theta^\star)\right)\bm{1}_{\{ \Tilde{\pi}^k = \pi^\star\}}
\end{equation}

Furthermore, using the fact that $\Tilde{\E}_k^\star[v(s_{t+1}; \theta^\star)|s_t, a_t] = \sum_{s'\in\cS}\theta^\star(s'|s_t, a_t) v(s';\theta^\star)$, we have
\begin{align}
    &\quad~\Tilde{\E}_k^\star[Z_k \bm{1}_{\{\Tilde{\pi}^k = \pi^\star\}}] \\
    &= \Tilde{\E}_k^\star\left[\sum_{t=t_k}^{t_{k+1}-1}\left( \sum_{s'\in\cS}\theta^\star(s'|s_t, a_t) v(s';\theta^\star) - v(s_t;\theta^\star) \right)\bm{1}_{\{\Tilde{\pi}^k = \pi^\star\}}\right]\\
    &= \Tilde{\E}_k^\star\left[\sum_{t=t_k}^{t_{k+1}-1}\left(\sum_{s'\in\cS}\theta^\star(s'|s_t, a_t) v(s';\theta^\star) - v(s_t;\theta^\star)\right) \right]\bm{1}_{\{\Tilde{\pi}^k = \pi^\star\}} \\
    &= \Tilde{\E}_k^\star\left[\sum_{t=t_k}^{t_{k+1}-1}\left( v(s_{t+1};\theta^\star) - v(s_t;\theta^\star)\right) \right]\bm{1}_{\{\Tilde{\pi}^k = \pi^\star\}} \\
    &= \Tilde{\E}_k^\star\left[v(s_{t_{k+1}};\theta^\star) - v(s_{t_k};\theta^\star)\right]\bm{1}_{\{\Tilde{\pi}^k = \pi^\star\}} \\
    &= \Tilde{\E}_k^\star\left[\left[v(s_{t_{k+1}};\theta^\star) - v(s_{t_k};\theta^\star)\right]\bm{1}_{\{\Tilde{\pi}^k = \pi^\star\}}\right]\label{eq:zkeq}
\end{align}

Taking the sum over all episodes $k$, we have
\begin{align}
    &\quad~\sum_{k=1}^K \E\left[Z_k \bm{1}_{\{\Tilde{\pi}^k = \pi^\star\}} \right]\\
    &=\E\left[\sum_{k=1}^K[v(s_{t_{k+1}};\theta^\star) - v(s_{t_k};\theta^\star)]\bm{1}_{\{\Tilde{\pi}^k = \pi^\star\}} \right]\\
    &=\E\left[v(s_{t_{K+1}};\theta^\star) \bm{1}_{\{\Tilde{\pi}^K = \pi^\star\}} - v(s_{t_{1}};\theta^\star) \bm{1}_{\{\Tilde{\pi}^1 = \pi^\star\}}\right]\\
    &\quad+ \E\left[\sum_{k=1}^{K-1}v(s_{t_{k+1}};\theta^\star)\bm{1}_{\{\Tilde{\pi}^k = \pi^\star, \Tilde{\pi}^{k+1} \neq \pi^\star\}} - \sum_{k=2}^{K}v(s_{t_{k}};\theta^\star)\bm{1}_{\{\Tilde{\pi}^{k-1} \neq \pi^\star, \Tilde{\pi}^{k} = \pi^\star\}} \right]\\
    &\leq \vbar + \vbar \E\left[\sum_{k=1}^{K-1} \bm{1}_{\{\Tilde{\pi}^k = \pi^\star, \Tilde{\pi}^{k+1} \neq \pi^\star\}} \right] \leq \vbar + \vbar\sum_{k=1}^{K-1}\Pr(\Tilde{\pi}^{k+1} \neq \pi^\star) \\
    &\leq \vbar + \vbar(K-1)\varepsilon \label{eq:sumzkeq}
\end{align}
where we used Lemma \ref{lem:psmonotone} in the last inequality.

In general, for $t_k\leq t < t_{k+1}$, since $a_t = \Tilde{\pi}^k(s_t)$, we have $v(s_t; \Tilde{\theta}^k) + \bar{r}^\star(\Tilde{\theta}^k) = r(s_t, a_t) + \sum_{s'\in\cS}\Tilde{\theta}^k(s'|s_t, a_t) v(s';\Tilde{\theta}^k)$ due to the assumption on the planning oracle $\texttt{MDPSolve}$. Therefore,
\begin{align}
    Z_k &= T_k \bar{r}^\star(\theta^\star) -  \sum_{t=t_k}^{t_{k+1}-1} \left(\bar{r}^\star(\Tilde{\theta}^k) + v(s_t; \Tilde{\theta}^k) - \sum_{s'\in\cS}\Tilde{\theta}^k(s'|s_t, a_t) v(s';\Tilde{\theta}^k) \right)\\
    &= \underbrace{T_k\left[\bar{r}^\star(\theta^\star) - \bar{r}^\star(\Tilde{\theta}^k)\right]}_{Z_{k, 0}} + \underbrace{\sum_{t=t_k}^{t_{k+1}-1}\left(\sum_{s'\in\cS} \theta^\star(s'|s_t, a_t) v(s';\Tilde{\theta}^k) - v(s_t;\Tilde{\theta}^k)\right)}_{Z_{k, 1}} \\
    &\quad + \underbrace{\sum_{t=t_k}^{t_{k+1}-1}\sum_{s'\in\cS}[\Tilde{\theta}^k(s'|s_t, a_t) - \theta^\star(s'|s_t, a_t)] v(s';\Tilde{\theta}^k)}_{Z_{k,2}}.\label{eq:zkineqdecomp}
\end{align}

We will analyze the three terms in \eqref{eq:zkineqdecomp} separately. First, we have
\begin{equation}\label{eq:zk0}
    \E_k[Z_{k, 0} \bm{1}_{\{\Tilde{\pi}^k \neq \pi^\star\}} ] = T_k\left(\E_k[\bar{r}^\star(\theta^\star)\bm{1}_{\{\Tilde{\pi}^k \neq \pi^\star\}}] - \E_k[\bar{r}^\star(\Tilde{\theta}^k)\bm{1}_{\{ \pi^\star \neq \Tilde{\pi}^k \}}] \right) = 0
\end{equation}
since $\theta^\star$ and $\Tilde{\theta}^k$ are i.i.d. given $\Tilde{\mathcal{D}}_{k}$. Following similar steps as \eqref{eq:zkeq} we have
\begin{equation}
    \Tilde{\E}_k^\star[Z_{k, 1}\bm{1}_{\{\Tilde{\pi}^k \neq \pi^\star\}}] = \Tilde{\E}_k^\star\left[v(s_{t_{k+1}};\Tilde{\theta}^k) - v(s_{t_k};\Tilde{\theta}^k)\right]\bm{1}_{\{\Tilde{\pi}^k \neq \pi^\star\}}\leq \vbar \bm{1}_{\{\Tilde{\pi}^k \neq \pi^\star\}}
\end{equation}
and hence  
\begin{equation}\label{eq:sumzk1}
\E[Z_{k, 1}\bm{1}_{\{\Tilde{\pi}^k \neq \pi^\star\}}]\leq \vbar \Pr(\Tilde{\pi}^k \neq \pi^\star) \leq \vbar\varepsilon.
\end{equation}
where we used Lemma \ref{lem:psmonotone}.
Combining \eqref{eq:sumzkeq}--\eqref{eq:sumzk1}, we have
\begin{align}
    \sum_{k=1}^K\E[Z_k] &= \sum_{k=1}^K\E[Z_k \bm{1}_{\{\Tilde{\pi}^k = \pi^\star\}}] + \sum_{k=1}^K\E[(Z_{k, 0} + Z_{k, 1} + Z_{k, 2}) \bm{1}_{\{\Tilde{\pi}^k \neq \pi^\star\}} ]\\
    &\leq \vbar + \vbar(K-1)\varepsilon + \sum_{k=1}^K\E[Z_{k,1} \bm{1}_{\{\Tilde{\pi}^k \neq \pi^\star\}}] + \sum_{k=1}^K\E[Z_{k, 2}\bm{1}_{\{\Tilde{\pi}^k \neq \pi^\star\}}]\\
    &\leq \vbar + \vbar(2K-1)\varepsilon + \sum_{k=1}^K\E[Z_{k, 2}\bm{1}_{\{\Tilde{\pi}^k \neq \pi^\star\}} ]\label{eq:sumzkbound}
\end{align}

For $S=1$, $Z_{k,2}=0$ and \eqref{eq:sumzkbound} proves the theorem; hence assume $S\geq 2$.

Now we bound the summation term in \eqref{eq:sumzkbound}. For convenience, define $I_k = \bm{1}_{\{\Tilde{\pi}^k \neq \pi^\star\}}$. First, since $v(s;\theta)\in [0, \vbar]$, we have
\begin{align}
    Z_{k, 2}&\leq \vbar \sum_{t=t_k}^{t_{k+1}-1}\sum_{s'\in\cS}(\Tilde{\theta}^k(s'|s_t, a_t) - \theta^\star(s'|s_t, a_t))_+\\
    &=\vbar  \sum_{t=t_k}^{t_{k+1}-1} \frac{1}{2}\|\Tilde{\theta}^k(\cdot|s_t, a_t) - \theta^\star(\cdot|s_t, a_t) \|_1\label{eq:zk2wp1}
\end{align}

Following \cite{osband2013more,ouyang2017learning}, define $N_t(s, a):=\sum_{\tau=0}^{t-1}\bm{1}_{\{(s_{\tau}, a_{\tau}) = (s, a) \}}$ to be the number of occurrences of the state-action pair $(s, a)$ from time $0$ to time $t-1$.

\textbf{Claim}: For any $k,t$ such that $t_k\leq t < t_{k+1}$ and any $\delta \in (0, 1)$,
\begin{align}
    &\quad~\E\left[\frac{1}{2}\|\Tilde{\theta}^k(\cdot|s_t, a_t) - \theta^\star(\cdot|s_t, a_t) \|_1\cdot I_k \right] \\
    &\leq  \E\left[\min\left\{ \sqrt{\dfrac{(2\log 2)S + 2\log(1/\delta)}{N_{t_k}(s_t, a_t) }}, 2\right\}\cdot I_k\right] + 2t_k SA \delta \label{eq:zk2concentration}
\end{align}

For $t_k\leq t < t_{k+1}$, define the event
\begin{equation}
    \mathcal{E}_t = \{N_{t}(s_t, a_t) > 2N_{t_k}(s_t, a_t) \}
\end{equation}

We can decompose the first term in \eqref{eq:zk2concentration} by
\begin{align}
    &\quad~\min\left\{ \sqrt{\dfrac{(2\log 2)S + 2\log(1/\delta)}{N_{t_k}(s_t, a_t) }}, 2\right\} I_k\\
    &\leq \min\left\{\sqrt{\dfrac{(2\log 2)S + 2\log(1/\delta)}{N_{t_k}(s_t, a_t) }}, 2\right\} I_k\bm{1}_{\mathcal{E}_t^c} + 2I_k\bm{1}_{\mathcal{E}_t}\\
    &\leq \min\left\{\sqrt{\dfrac{(4\log 2)S + 4\log(1/\delta)}{N_{t}(s_t, a_t) }}, 2\right\} I_k\bm{1}_{\mathcal{E}_t^c} + 2I_k\bm{1}_{\mathcal{E}_t}\\
    &\leq \underbrace{\min\left\{\sqrt{\dfrac{(4\log 2)S + 4\log(1/\delta)}{\max\{N_{t}(s_t, a_t), 1\} }},2\right\}}_{\beta_t} I_k + 2I_k\bm{1}_{\mathcal{E}_t}\label{eq:zk2decompose}
\end{align}

Combining \eqref{eq:zk2wp1}--\eqref{eq:zk2decompose} with $\delta=(SAT^2)^{-1}$ and using $\sum_{k=1}^K T_k t_k\leq T^2$, we have
\begin{align}
    &\quad~\sum_{k=1}^K\E[Z_{k, 2}I_k ] \\
    &\leq \vbar\E\left[\sum_{k=1}^K \sum_{t=t_k}^{t_{k+1}-1} \beta_t I_k \right] + 2\vbar \E\left[\sum_{k=1}^K \sum_{t=t_k}^{t_{k+1}-1} I_k \bm{1}_{\mathcal{E}_t} \right] + 2\vbar SA\delta \sum_{k=1}^K T_k t_k \\
    &\leq \vbar\E\left[\sum_{k=1}^K \sum_{t=t_k}^{t_{k+1}-1} \beta_t I_k \right] + 2\vbar \E\left[\sum_{k=1}^K \sum_{t=t_k}^{t_{k+1}-1} I_k \bm{1}_{\mathcal{E}_t} \right] + 2\vbar
    \label{eq:zk2threeterm}
\end{align}

By the Cauchy-Schwarz inequality, we have
\begin{align}
    &\quad~\E\left[\sum_{k=1}^K \sum_{t=t_k}^{t_{k+1}-1} \beta_t I_k \right]\\
    &\leq \sqrt{\E\left[\sum_{k=1}^K \sum_{t=t_k}^{t_{k+1}-1} I_k^2 \right]} \sqrt{\E\left[ \sum_{k=1}^K\sum_{t=t_k}^{t_{k+1}-1} \beta_t^2\right] } \\
    &=\sqrt{\sum_{k=1}^K \sum_{t=t_k}^{t_{k+1}-1} \E[I_k] } \sqrt{\E\left[\sum_{t=0}^{T-1} \beta_t^2\right] } \stackrel{(a)}{\leq} \sqrt{\varepsilon T} \sqrt{ \E\left[\sum_{t=0}^{T-1} \beta_t^2\right] }
    \label{eq:zk2term1cs}
\end{align}
where in $(a)$ we have used Lemma \ref{lem:psmonotone}.

Furthermore, we have
\begin{align}
    &\quad~\sum_{t=0}^{T-1}\frac{1}{\max\{N_t(s_t,a_t),1\}}
    =\sum_{(s,a)\in\cS\times\cA}\sum_{j=0}^{N_T(s,a)-1}\frac{1}{\max\{j,1\}}\\
    &\leq \sum_{(s,a)\in\cS\times\cA}\left[2+\log(1+N_T(s,a))\right]
    \leq SA\left[2+\log\left(1+\frac{T}{SA}\right)\right].
    \label{eq:zk2term1hsum}
\end{align}
where the last inequality is due to the concavity of the function $f(x) = \log(x+1)$. Since $\beta_t\leq 2$, the first two visits to each pair contribute at most $8\leq(4\log 2)S+4\log(SAT^2)$ to $\sum_t\beta_t^2$, and the remaining visits use the same harmonic-tail bound.

Fix a state-action pair $(s, a)\in\cS\times\cA$, let $k_1 < k_2 < ... < k_m$ be the first $m$ episodes such that the event $\{N_{t}(s, a) > 2N_{t_k}(s, a)\}$ occurs.
Let $\Bar{\tau}_l := t_{k_l + 1} - 1$ denote the last timestamp in episode $k_l$. Using induction, it can be shown that $N_{\Bar{\tau}_l}(s, a) \geq 2^l - 1$ for $1\leq l\leq m$. Therefore, the above event can occur in no more than $\log_2(N_{T}(s, a)+1)$ episodes. Therefore, we have
\begin{align}
    \sum_{k=1}^K \sum_{t=t_k}^{t_{k+1}-1} \bm{1}_{\mathcal{E}_t} &\leq \sum_{(s, a)\in\cS\times\cA} \left(\max_{k\in[K]} T_k\right) \log_2(N_{T}(s, a)+1)  \\
    &\leq SA\left(\max_{k\in[K]} T_k\right) \log_2\left(\frac{T}{SA}+1\right).
\end{align}
where the last inequality is due to the concavity of the function $f(x) = \log(x+1)$.

Also note that one can obtain a prior-dependent bound on the second term in \eqref{eq:zk2threeterm} with Lemma \ref{lem:psmonotone}:
\begin{align}
    \E\left[\sum_{k=1}^K \sum_{t=t_k}^{t_{k+1}-1} I_k \bm{1}_{\mathcal{E}_t} \right]\leq \E\left[\sum_{k=1}^K \sum_{t=t_k}^{t_{k+1}-1} I_k \right] =\sum_{k=1}^K \sum_{t=t_k}^{t_{k+1}-1}  \E\left[I_k \right] \leq \varepsilon T
\end{align}

Therefore, we conclude that
\begin{align}
    \E\left[\sum_{k=1}^K \sum_{t=t_k}^{t_{k+1}-1} I_k \bm{1}_{\mathcal{E}_t} \right]\leq \min\left\{\varepsilon T, SA\left(\max_{k\in[K]} T_k\right)\log_2\left(\frac{T}{SA}+1\right) \right\}\label{eq:zk2term2}
\end{align}

Combining \eqref{eq:sumzkbound}, \eqref{eq:zk2threeterm}, \eqref{eq:zk2term1cs}, and \eqref{eq:zk2term2} with the squared-radius bound above, and using $(4\log 2)S+4\log(SAT^2)\leq 4S\log(2SAKT)$ for $S\geq 2$ and $K\geq 1$, we obtain
\begin{align}
    &\quad~\sum_{k=1}^K\E[Z_{k} ] \\
    &\leq \vbar + \vbar\varepsilon(2K-1) + \vbar\sqrt{\varepsilon T\left[(4\log 2)S+4\log(SAT^2)\right]SA\left[1+\log\left(1+\frac{T}{SA}\right)\right]} \\
    &+ 2\vbar \min\left\{\varepsilon T, SA\left(\max_{k\in[K]} T_k\right) \log_2\left(\frac{T}{SA}+1\right) \right\} + 2\vbar \\
    &\leq 3\vbar + 2\vbar (R_1 + R_2 + R_3)
\end{align}
where
\begin{align}
    R_1 &:= \varepsilon K \\
    R_2 &:= \sqrt{\varepsilon S^2 A T \log(2SAKT)\left[1+\log\left(\frac{T}{SA}+1\right)\right] } \\
    R_3 &:= \min\left\{\varepsilon T, SA\left(\max_{k\in[K]} T_k\right)\log_2\left(\frac{T}{SA}+1\right) \right\}
\end{align}

%\textcolor{blue}{Set $T_k \sim \varepsilon k$, then $K = O(\sqrt{T/\varepsilon})$ and $\max_{k\in[K]}T_k = O(\sqrt{\varepsilon T})$. In this case, $R_1 = O(\sqrt{\varepsilon T}), R_2 = \Tilde{\order}(\sqrt{\varepsilon S^2 A T}), R_3 = \Tilde{\order}(SA \sqrt{\varepsilon T})$. As a result, the regret is bounded by $3\vbar + 2\vbar\cdot \Tilde{\order}(SA \sqrt{\varepsilon T}) $.}

\begin{proof}[Proof of Claim]
    Let $\check{\theta}^{n}(\cdot|s, a)$ be the empirical distribution of the next states of the state-action pair $(s, a)$ after observing $(s, a)$ exactly $n$ times. By the $L_1$ concentration inequality for empirical distributions \citep{weissman2003inequalities}, we have
    \begin{align}
        \Pr_{\theta^\star}(\|\check{\theta}^n(\cdot|s, a) - \theta^\star(\cdot|s,a) \|_1 > \lambda) \leq 2^{S}\exp\left(-\frac{1}{2}n \lambda^2\right)
    \end{align}
    for any $\lambda > 0$. For $\delta > 0$, define $\xi(n, \delta) := \min\left\{\sqrt{\frac{(2\log 2)S + 2\log(1/\delta)}{n}}, 2\right\}$, we have
    \begin{align}
        \Pr_{\theta^\star}\left(\|\check{\theta}^n(\cdot|s, a) - \theta^\star(\cdot|s,a) \|_1 > \xi(n, \delta) \right) \leq \delta.
    \end{align}

    Define 
    \begin{align}
        \hat{\theta}^{t}(\cdot|s, a):= \begin{cases}
        \frac{1}{N_t(s,a)}\sum_{\tau=0}^{t-1} \bm{1}_{\{(s_{\tau}, a_{\tau}) = (s, a) , s_{\tau+1} = \cdot \} } &N_t(s,a)>0 \\
        \text{uniform}&N_t(s,a)=0
        \end{cases}
    \end{align}
    to be the empirical distribution of next states of $(s, a)$ in the first $t$ steps. Define
    \begin{align*}
        \hat{\varTheta}^k &= \left\{\theta: \|\hat{\theta}^{t_k}(\cdot|s, a) - \theta(\cdot|s,a) \|_1\leq \xi(N_{t_k}(s, a), \delta)~~\forall s\in\mathcal{S}, a\in\mathcal{A} \right\}
    \end{align*}

    Then, we have $\Pr(\theta^\star\not\in \hat{\varTheta}^k) = 0$ if $t_k = 0$ (since $\xi(0,\delta) = 2$) and
    \begin{align}
        &\quad~\Pr(\theta^\star\not\in \hat{\varTheta}^k) \\
        &\leq \Pr\left(\exists n \in [t_k], s\in\mathcal{S}, a\in\mathcal{A}~~ \|\check{\theta}^n(\cdot|s, a) - \theta^\star(\cdot|s,a) \|_1 > \xi(n, \delta) \right)\\
        &\leq t_k SA \delta 
    \end{align}
    if $t_k > 0$.

    Since $\theta^\star$ and $\Tilde{\theta}^k$ are i.i.d. given $\Tilde{\mathcal{D}}_k$, and the random set $\hat{\varTheta}^k$ is measurable with respect to $\Tilde{\mathcal{D}}_k$, we have $\Pr(\theta^\star\not\in \hat{\varTheta}^k) = \Pr(\Tilde{\theta}^k\not\in \hat{\varTheta}^k)$. Therefore we have
    \begin{align}
        &\quad~\E\left[\|\Tilde{\theta}^k(\cdot|s_t, a_t) - \theta^\star(\cdot|s_t, a_t) \|_1\cdot I_k  \right]\\
        &\leq \E\left[\|\Tilde{\theta}^k(\cdot|s_t, a_t) - \theta^\star(\cdot|s_t, a_t) \|_1\cdot I_k \bm{1}_{\{\theta^\star, \Tilde{\theta}^k \in \hat{\varTheta}^k \}}\right] \\
        &+ \E\left[\|\Tilde{\theta}^k(\cdot|s_t, a_t) - \theta^\star(\cdot|s_t, a_t) \|_1 \cdot I_k \left(\bm{1}_{\{ \theta^\star \not\in \hat{\varTheta}^k\}} + \bm{1}_{\{ \Tilde{\theta}^k \not\in \hat{\varTheta}^k\}}\right) \right] \\
        &\leq \E[2\xi(N_{t_k}(s_t, a_t), \delta)\cdot I_k] + 2\left(\Pr(\theta^\star \not\in \hat{\varTheta}^k) + \Pr(\Tilde{\theta}^k \not\in \hat{\varTheta}^k) \right)\\
        &\leq 2\E[\xi(N_{t_k}(s_t, a_t), \delta)\cdot I_k] + 4t_k SA \delta 
    \end{align}

\end{proof}

\section{Proof of Lemma \ref{lem:estimator2p1}}
\label{app:proof-estimator2p1}

If $\Tilde{\pi}^1 \neq \pi^\star$, then either $\hat{\pi}^\star \neq \Tilde{\pi}^1$ or $\hat{\pi}^\star \neq \pi^\star$ must be true. Conditioning on $\mathcal{D}_0$, $\Tilde{\pi}^1$ is identically distributed as $\pi^\star$ while $\hat{\pi}^\star$ is deterministic, therefore
\begin{align*}
    \Pr(\Tilde{\pi}^1 \neq \pi^\star)\leq \Pr(\hat{\pi}^\star \neq \Tilde{\pi}^1) + \Pr(\hat{\pi}^\star \neq \pi^\star) =  2 \Pr(\hat{\pi}^\star \neq \pi^\star).
\end{align*}

\section{Proof of Lemma \ref{lem:piestimator}}
\label{app:proof-piestimator}
In this proof, let $\Pr_{\theta^\star}$ be the distribution induced by the expert's policy $\pi^\beta$ conditioned on $\theta^\star$.

Under Assumption \ref{assump:span}, we have $q(s, a';\theta) - q(s, a;\theta) \leq r(s, a') - r(s, a) + \vbar \leq 1 + \vbar$ for any $s\in\cS, a, a'\in\cA$. Therefore, for any $s\in\cS, a\in\cA$, we have
\begin{align}
    \pi^\beta(a|s) &:= \dfrac{\exp(\beta q(s, a; \theta^\star))}{\sum_{a'\in\cA}\exp(\beta q(s, a'; \theta^\star)) }\\
    &= \dfrac{1}{\sum_{a'\in\cA}\exp[\beta (q(s, a'; \theta^\star) - q(s, a;\theta^\star) )]} \\
    &\geq \dfrac{1}{1 + (A-1)\exp(\beta(1+\vbar))}=:\pi_{\min}\label{eq:minactprob}
\end{align}

Combining \eqref{eq:minactprob} with Assumption \ref{assump:minprob}, for any $s, s'\in \mathcal{S}$,
\begin{equation}
    \Pr_{\theta^\star}(\exists 0 < k \leq S, s_k = s'|s_0 = s)\geq \delta (\pi_{\min})^{S} =:\Tilde{\delta}
\end{equation}

Let $X_{s,s'}$ be a random variable that represents the hitting time of $s'$ starting from $s$, i.e. it follows the conditional distribution of $\min\{t>0: s_t=s'\}$ given $s_0=s$. Following the same derivation of Equation (1.18) of \cite{levin2017markov}, we have
    \begin{equation}
        \Pr_{\theta^\star}(X_{s,s'} > kS) \leq (1-\Tilde{\delta})^k\quad\forall k\in\mathbb{N}
    \end{equation}

    Therefore, for any $u > 0$,
    \begin{align}
        \Pr(X_{s,s'} > u) \leq (1-\Tilde{\delta})^{\lfloor u/S\rfloor} \leq (1-\Tilde{\delta})^{\frac{u}{S}-1}
    \end{align}

    Set $\lambda = -\frac{1}{2S}\log(1-\Tilde{\delta}) > 0$, we have
    \begin{align}
        \E[\exp(\lambda X_{s,s'})] &= \int_0^\infty \Pr(\exp(\lambda X_{s,s'}) > x )\mathrm{d} x\\
        &=1 + \int_1^\infty \Pr(\exp(\lambda X_{s,s'}) > x )\mathrm{d} x\\
        &=1 + \lambda\int_0^\infty \Pr(\exp(\lambda X_{s,s'}) > \exp(\lambda u))  \exp(\lambda u) \mathrm{d} u\\
        &\leq 1 + \lambda\int_0^\infty (1-\Tilde{\delta})^{\frac{u}{S}-1}\exp(\lambda u) \mathrm{d} u\\
        &= 1 + \lambda\int_0^\infty (1-\Tilde{\delta})^{\frac{u}{2S}-1}\mathrm{d} u =:\exp(\Tilde{\lambda}).
    \end{align}

    Next, we establish that with high probability, each state must be visited at least $c_1 N$ times in an $N$-step trajectory for some constant $c_1 > 0$. Let $\Bar{N}(s):= \sum_{t=0}^{N-1} \bm{1}_{\{\Bar{s}_t = s\}} $ be the number of visits to state $s$ in the offline data (excluding the final state $\bar{s}_N$). Let $X_{\bar{s}_0, s; 0}, X_{s,s; 1}, X_{s,s; 2}, X_{s,s; 3},\cdots$ be mutually independent random variables representing hitting times. For any $K \in\mathbb{N}$ we have
    \begin{align}
        \Pr_{\theta^\star}(\Bar{N}(s) < K) &\leq \Pr_{\theta^\star}\left(X_{\bar{s}_0, s; 0} + \sum_{k=1}^{K-1} X_{s,s; k} \geq N \right)\\
        &= \Pr_{\theta^\star}\left(\exp\left(\lambda X_{\bar{s}_0, s; 0} + \lambda\sum_{k=1}^{K-1} X_{s,s; k}\right) \geq \exp(\lambda N) \right)\\
        &\leq \exp(-\lambda N)\E\left[\exp\left(\lambda X_{\bar{s}_0, s; 0} + \lambda\sum_{k=1}^{K-1} X_{s,s; k}\right)\right]\\
        &\leq \exp(-\lambda N)[\exp(\Tilde{\lambda})]^K = \exp(-(\lambda N - \Tilde{\lambda} K)).
    \end{align}

    Set $K = \lfloor \frac{\lambda}{2\Tilde{\lambda}} N \rfloor$, we have $\Pr_{\theta^\star}(\Bar{N}(s) < K) \leq \exp(-\lambda N / 2)$. 

    Fix $s\in\cS$ and let $a_s^\star:=\pi^{*}(s)$. Under Assumption \ref{assump:gap} we have
    \begin{align}
        \pi^\beta(a_s^\star|s) &= \dfrac{\exp(\beta q(s, a_s^\star; \theta^\star))}{\sum_{a\in\cA}\exp(\beta q(s, a; \theta^\star))} \\
        &= \dfrac{1}{\sum_{a\in\cA}\exp[\beta (q(s, a; \theta^\star) - q(s, a_s^\star; \theta^\star) ) ]}\\
        &\geq \dfrac{1}{1 + (A-1)\exp(-\beta\Delta)}.
    \end{align}

    For the rest of the proof, we assume that $\beta \geq \log(2(A-1))/\Delta$. Then $\pi^\beta(a_s^\star|s) \geq \frac{2}{3}$.

    For $L\in\mathbb{N}$, define $\mathcal{F}_L(s)$ to be the event that in the first $L$ visits to the state $s$, the action $\pi^\star(s)$ is chosen $\leq L/2$ times. We have
    \begin{equation}\label{eq:FL}
        \Pr_{\theta^\star}(\mathcal{F}_L(s)) \leq \Pr(\mathcal{B}(L, 2/3) \leq L/2) \leq \exp\left(-\frac{L}{18}\right)
    \end{equation}
    where $\mathcal{B}(L, p)$ is a binomial random variable with parameters $L$ and $p$. The last inequality in \eqref{eq:FL} can be established with Hoeffding's inequality.

    Now, consider a construction of $\hat{\pi}^\star$ as follows: For each $s\in\cS$ and $a\in\cA$, let $\Bar{N}(s, a):= \sum_{t=0}^{N-1} \bm{1}_{\{\Bar{s}_t = s, \Bar{a}_t = a \}} $. Define $\hat{\pi}^\star(s)=\argmax_{a\in\cA}\Bar{N}(s, a)$ (ties are broken following a fixed ordering on $\cA$). Now we have
    \begin{align}
        \Pr_{\theta^\star}(\hat{\pi}^\star(s)\neq \pi^\star(s)) &\leq \Pr_{\theta^\star}(\Bar{N}(s, a_s^\star) \leq \Bar{N}(s)/2) \\
        &\leq \Pr_{\theta^\star}(\Bar{N}(s) < K) + \Pr_{\theta^\star}(\Bar{N}(s) \geq K, \Bar{N}(s, a_s^\star) \leq \Bar{N}(s)/2 ) \\
        &\leq \Pr_{\theta^\star}(\Bar{N}(s) < K) + \sum_{L = K}^N \Pr_{\theta^\star}(\mathcal{F}_L(s))\\
        &\leq \exp\left(-\dfrac{\lambda N}{2}\right) + N \exp\left(-\frac{K}{18}\right) \leq \left(1 + e^{1/18} N \right)\exp(-c N)
    \end{align}
    where $c:=\min\{\frac{\lambda}{2}, \frac{\lambda}{36\Tilde{\lambda}} \} $.

    Note that $c$ is a function of $(S, A, \vbar, \Delta, \delta, \beta)$, in particular, it is independent of $\theta^\star$. Therefore we conclude that
    \begin{align}
        \Pr(\hat{\pi}^\star\neq \pi^\star) \leq \sum_{s\in\cS} \Pr(\hat{\pi}^\star(s)\neq \pi^\star(s)) \leq S(1 + 1.06 N) \exp(-c N)
    \end{align}

% \section{Empirical Results}% % 

\section{Experiment details}
\label{app:experiments}

\subsection{Environment specifications}
\label{app:envs}

All nine benchmarks are communicating tabular MDPs. The optimal
average reward used to compute regret was obtained by relative value
iteration on the ground-truth parameters.

\paragraph{RiverSwim.}
As illustrated in Figure~\ref{fig:env_diagrams_riverswim}, \textit{RiverSwim} is a classical hard-exploration environment with $A=2$ and configurable state space. We consider $S \in \{6, 25, 50, 100\}$, where higher $S$ corresponds to larger exploration difficulty.

\begin{figure}[ht]
\centering
% Source: figures/riverswim_diagram.tex (compile standalone, then
% \includegraphics; or load \usepackage{standalone} in the preamble
% and uncomment the \input below to render inline).
% \input{figures/riverswim_diagram.tex}
\includegraphics[width=0.85\linewidth]{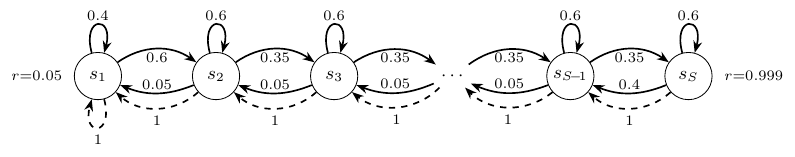}
\caption{The $S$-state \textit{RiverSwim} MDP. Solid arrows depict the
{right} action and dashed arrows depict the
{left} action. The large reward
$r=0.999$ is given in $s_S$ (for all actions); the
small reward $r=0.05$ is collected at $s_1$  (for all actions).}
\label{fig:env_diagrams_riverswim}
\end{figure}

\paragraph{Grid4Room and Grid2Room.}
\textit{Grid4Room} has $S=20, A=4$ and \textit{Grid2Room} has $S=55, A=4$. The room
layouts and transition model are shown in
Figure~\ref{fig:env_diagrams_grid}.

\begin{figure}[h]
\centering
\begin{subfigure}[b]{0.42\linewidth}
\centering
\includegraphics[height=0.85\linewidth, trim={32mm 0 32mm 10mm}, clip]{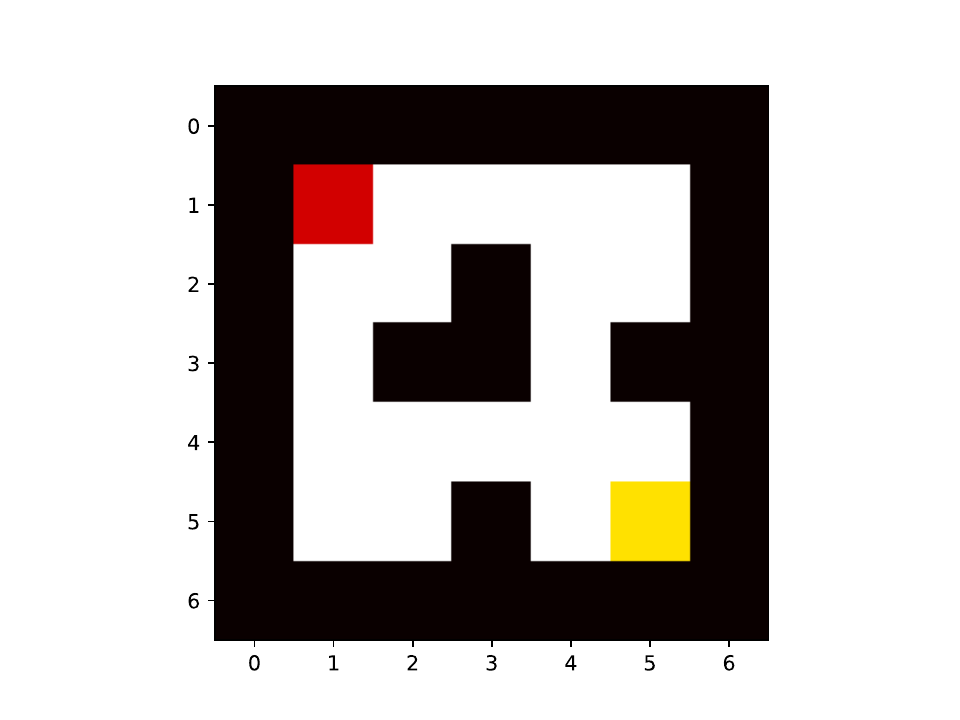}
\caption{\textit{Grid4Room} ($S=20$).}
\label{fig:grid4room_diagram}
\end{subfigure}\hfill
\begin{subfigure}[b]{0.42\linewidth}
\centering
\includegraphics[height=0.85\linewidth, trim={20mm 0 22mm 10mm}, clip]{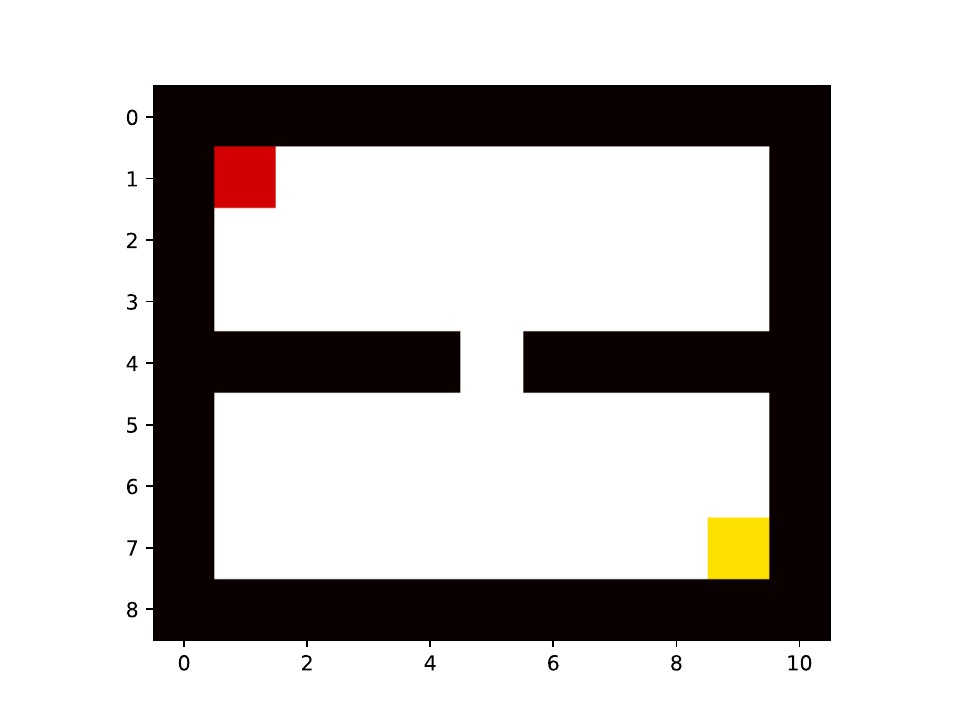}
\caption{\textit{Grid2Room} ($S=55$).}
\label{fig:grid2room_diagram}
\end{subfigure}
\caption{Gridworld layouts (from
\citep{bourel2020tightening}). White cells are passable, black cells
are walls; the marked cell yields a positive reward. Transitions move
as intended w.p. $0.7$, to each perpendicular neighbor w.p. $0.1$, and stay in place otherwise.}
\label{fig:env_diagrams_grid}
\end{figure}

\paragraph{ClusterRing.}
We introduce \textit{ClusterRing}, a sparse-reward family of MDPs that require moderate exploration. We consider $M \in \{5, 10, 20\}$, which corresponds to $S = 10M$, and $A = 10$. Each
instance is a fixed random MDP organized as $M$ ordered clusters on a
directed ring: (i) within-cluster transitions stay inside the cluster
except at one designated exit state; (ii) at the exit, a particular
action transitions stochastically to the entrance of cluster
$m+1 \pmod M$; (iii) small rewards are placed at every entrance and
exit, and a large reward at the goal-cluster exit. The optimal policy
must traverse the entire ring once, so finding it requires $\Omega(M)$
exploration depth. A schematic for $M=3$ appears in
Figure~\ref{fig:clusterring_diagram}.

\begin{figure}[h]
\centering
\includegraphics[width=0.45\linewidth]{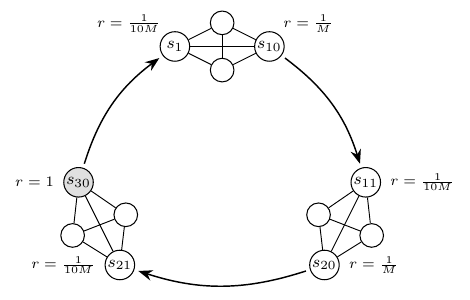}
\caption{\textit{ClusterRing}-$M$ schematic for $M=3$. Each cluster contains
$10$ states (representative states shown).}
\label{fig:clusterring_diagram}
\end{figure}

\subsection{Practical implementation for informed RLSVI}\label{app:practical-alg}

Algorithm~\ref{app-alg:inf_irlsvi} gives self-contained, practical pseudocode of informed RLSVI (iRLSVI) used in our experiments. Notice that in Eq~\eqref{eq:lazy-rlsvi-loss}, we employ the \emph{frozen-target} trick similar to the frozen target-network in DQN \citep{mnih2015human}. We also note that under frozen-targets, if $\Sigma_0 = \sigma_0 I$ for some $\sigma_0 > 0$, then the update step for uRLSVI and piRLSVI reduces to a ridge regression problem which admits a closed-form solution. In our experiments, we simply use $\Sigma_0 = I$.

% Recall that we use the
% convention $\check s_h = \bar s_h$ if $h < N$ and
% $\check s_h = s_{h-N}$ otherwise; $\check a_h$ is defined similarly.

\begin{algorithm}[t]
\caption{inf-iRLSVI with the frozen-target trick}
\label{app-alg:inf_irlsvi}
\begin{algorithmic}[1]
\STATE \textbf{Input:} Offline data $\mathcal{D}_0 = (\bar{s}_0, \bar{a}_0, ..., \bar{s}_{N})$; Episode schedule $\{T_k\}_{k=1}^K$; TD-error prior variance $\sigma^2$; Parameter prior covariance $\Sigma_0$; Expert-deliberateness prior rate $\lambda_2$; EMA rates $\eta_t$
\STATE Initialize $t_1 \gets 0$, $\theta^0 \sim \mathcal{N}(0, \Sigma_0)$, $\bar\theta \gets \theta^0$, $\beta^0 \sim \textrm{Exp}(\lambda_2)$, and $\tilde r_{0} = \sum_{n=0}^{N-1} r(\bar{s}_n, \bar{a}_n) / N$.
\FOR{$k=1$ to $K$}
    \STATE $t_{k+1} \gets t_k + T_k$
    \STATE Sample $z_d^k \sim \mathcal N(0,\sigma^2)$ for each transition $d=(s_d,a_d,r_d,s'_d)\in\mathcal D_0\cup\mathcal H_{t_k}$
    \STATE Sample $w_n^k \sim \mathrm{Exp}(1)$ for each offline pair $(\bar s_n,\bar a_n)\in\mathcal D_0$
    \STATE Sample $\tilde\theta^k \sim \mathcal N(0,\Sigma_0)$
    \STATE Define
    \begin{align}
        \widetilde{\mathcal L}_{\mathrm{RLSVI'}}^k(\theta) &= \frac{1}{2\sigma^2} \sum_{d\in \mathcal D_0 \cup \mathcal H_{t_k}} \left[q(s_d,a_d;\theta)- \left(r_d + z_d^k + \max_{a'} q(s_d',a';\bar\theta) - \tilde r_{t_k}\right)\right]^2 \label{eq:lazy-rlsvi-loss} \\
        \widetilde{\mathcal L}_{\mathrm{IL}}^k(\theta,\beta) &= \sum_{n=0}^{N-1} w_n^k \left[ \log\sum_{a'} e^{\beta q(\bar s_n,a';\theta)} - \beta q(\bar s_n,\bar a_n;\theta) \right] \\
        \widetilde{\mathcal L}_{\mathrm{prior}}^k(\theta, \beta) &= \frac{1}{2}(\theta-\tilde\theta^k)^\top \Sigma_0^{-1}(\theta-\tilde\theta^k) + \lambda_2\beta
    \end{align}
    \STATE Update:
        \[
        (\theta^{k},\beta^{k})\in\arg\min_{\theta,\beta}
        \Bigl\{
        \widetilde{\mathcal L}_{\mathrm{RLSVI'}}^k(\theta)
        +
        \widetilde{\mathcal L}_{\mathrm{IL}}^k(\theta,\beta)
        +
        \widetilde{\mathcal L}_{\mathrm{prior}}^k(\theta, \beta)
        \Bigr\}
        \]
    \FOR{$t=t_k$ to $t_{k+1}-1$}
        \STATE Take action $a_t \in \argmax_a q(s_t,a;\theta^{k})$
        \STATE Update $\tilde r_{t+1} \gets (1- \eta_t) \tilde{r}_{t} + \eta_t r(s_t, a_t)$
        \STATE At target update steps, $\bar\theta \gets \theta^k$.
    \ENDFOR
\ENDFOR
\end{algorithmic}
\end{algorithm}

\subsection{Hyperparameters}\label{app:hparams}

The hyperparameters for iRLSVI are specified in Table~\ref{tab:hparams}. Such hyperparameters are shared across all environments, with no additional tuning efforts. The online-only baselines inherit hyperparameters from the default configurations in the IMED-RL implementation \citep{pesquerel2022imed}.

\begin{table}[h]
\small
\centering
\caption{Hyperparameters for iRLSVI and variants. The same values are used across all
environments.}
\label{tab:hparams}
\begin{tabular}{ll}
\toprule
Parameter & Value \\
\midrule
% Online steps ($T$)              & $100,000$ \\
Episode length ($T_k$)          & $100$ \\
Target update interval & $1{,}000$ \\
Optimizer for iRLSVI update & L-BFGS-B (200 iters) \\
TD-error prior variance ($\sigma^2$)             & $1.0$ \\
Parameter prior covariance ($\Sigma_0$) & $I$\\
Expert deliberateness prior rate ($\lambda_2$)            & 0.01 \\
EMA rate for $\tilde r_t$ ($\eta_t$)                  & $0.01$ \\
\bottomrule
\end{tabular}
\end{table}

\paragraph{Update frequency.} We note that the three algorithms have different update frequencies: iRLSVI re-plans on a fixed schedule of $T_k = 100$ steps via L-BFGS-B, IMED-RL recomputes value
iteration and the IMED index every step, and UCRL3 re-plans only at
doubling-trick episode boundaries via extended value iteration (EVI).

\subsection{Statistical significance and reproducibility}\label{app:variance}
\ydzedit{Tables~\ref{tab:main_results} and~\ref{tab:prior-main}, Figure~\ref{fig:regret_vs_beta_main}, and the additional experiments in Appendices~\ref{app:robustness}--\ref{app:episode-schedule} use $20$ seeds per configuration. The figures in Appendix~\ref{app:beta_grid} use $10$ seeds; the learning curves in Appendix~\ref{app:percurves} use $20$ seeds. The offline dataset is regenerated per seed but is shared across all agents within a seed, so that the same dataset realization is given to every method in any one trial. Error bars and bands show 95\% confidence intervals across seeds. The 20-seed intervals are two-sided Student-$t$ intervals. For piUCRL3 on \textit{ClusterRing}-10 ($\beta=3$, $\kappa=10$), one complete learning curve was recovered from its original training log after an interrupted upload.}

\subsection{Compute}\label{app:compute}

All experiments run on CPU-only nodes, e.g., $8$ CPU cores and $16$ GB memory suffice for fast execution of our experiment scripts. Wall-clock per seed at $T = 100k$ ranges from seconds
(\textit{RiverSwim}-6) to a few minutes (\textit{ClusterRing}-20) for
iRLSVI; UCRL3's per-call EVI cost grows with $|\mathcal S|^2|\mathcal A|$
and exceeds our wall-clock budget on \textit{ClusterRing}-20.

\subsection{Regret vs.\ expert deliberateness and data ratio}\label{app:beta_grid}

\ydzedit{Figures~\ref{fig:app_regret_vs_beta_riverswim}, \ref{fig:app_regret_vs_beta_clusterring}, and \ref{fig:app_regret_vs_beta_grid} extend Figure~\ref{fig:regret_vs_beta_main} to all benchmarks. Across the \textit{RiverSwim} and \textit{ClusterRing} families, iRLSVI's regret decreases monotonically in $\beta$ at every $\kappa \in \{1,5,10\}$, and the iRLSVI--piRLSVI gap widens with $\beta$. In \textit{Grid4Room} and \textit{Grid2Room}, iRLSVI's regret generally decreases with $\beta$, while piIMED-RL maintains low, approximately constant regret across the tested range. This may reflect IMED-RL's per-step index updates being well matched to small gridworlds (cf.\ Section~\ref{sec:experiments}).}

\begin{figure}[t]
\centering
\includegraphics[width=\linewidth]{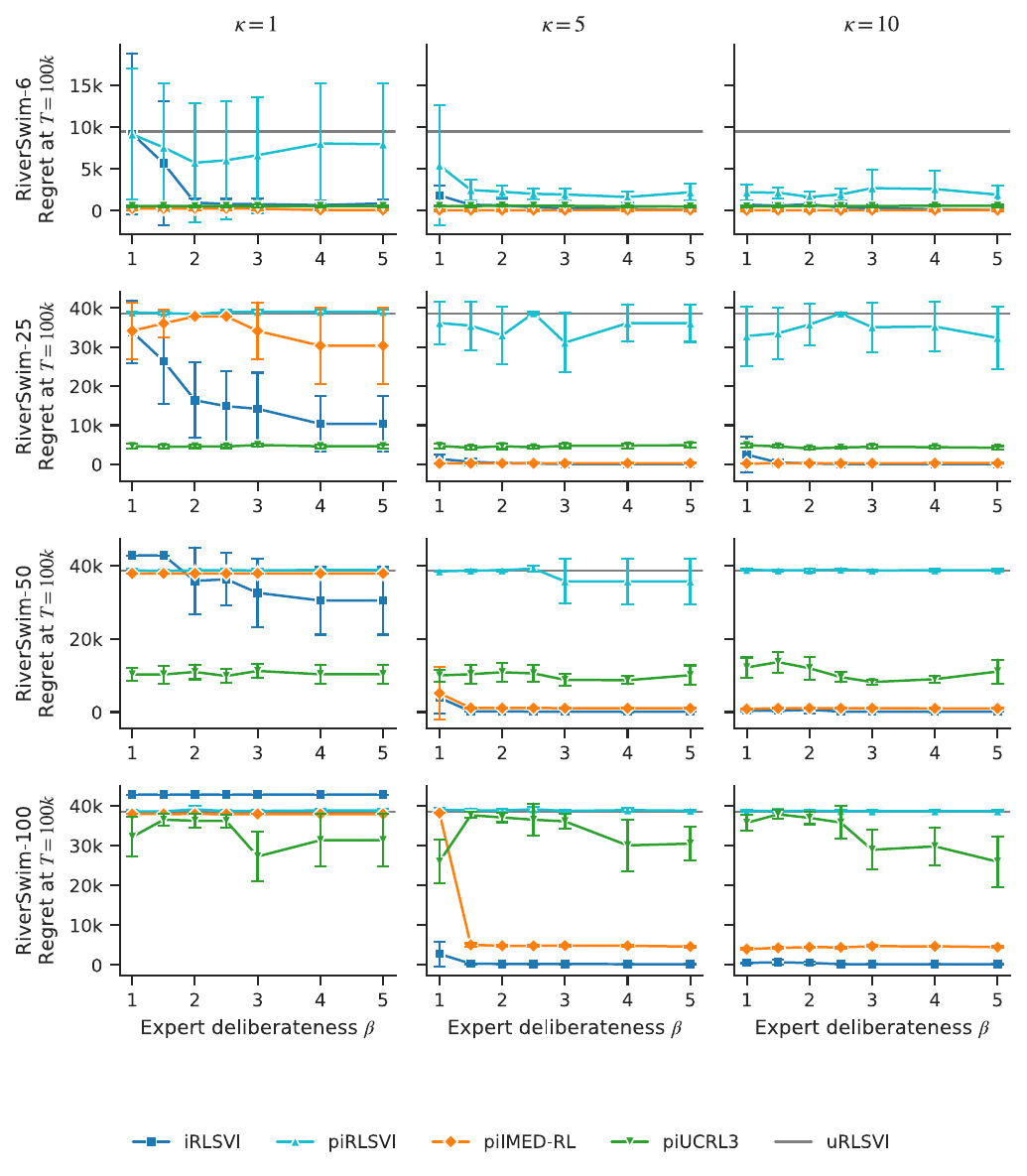}
\caption{Cumulative regret at $T=100{,}000$ online steps versus expert deliberateness $\beta$ for the \textit{RiverSwim} family.
Rows: environments; columns: data ratio $\kappa \in \{1, 5, 10\}$. Results are averaged over 10 seeds. Error bars show 95\% confidence intervals.}
\label{fig:app_regret_vs_beta_riverswim}
\end{figure}

\begin{figure}[t]
\centering
\includegraphics[width=\linewidth]{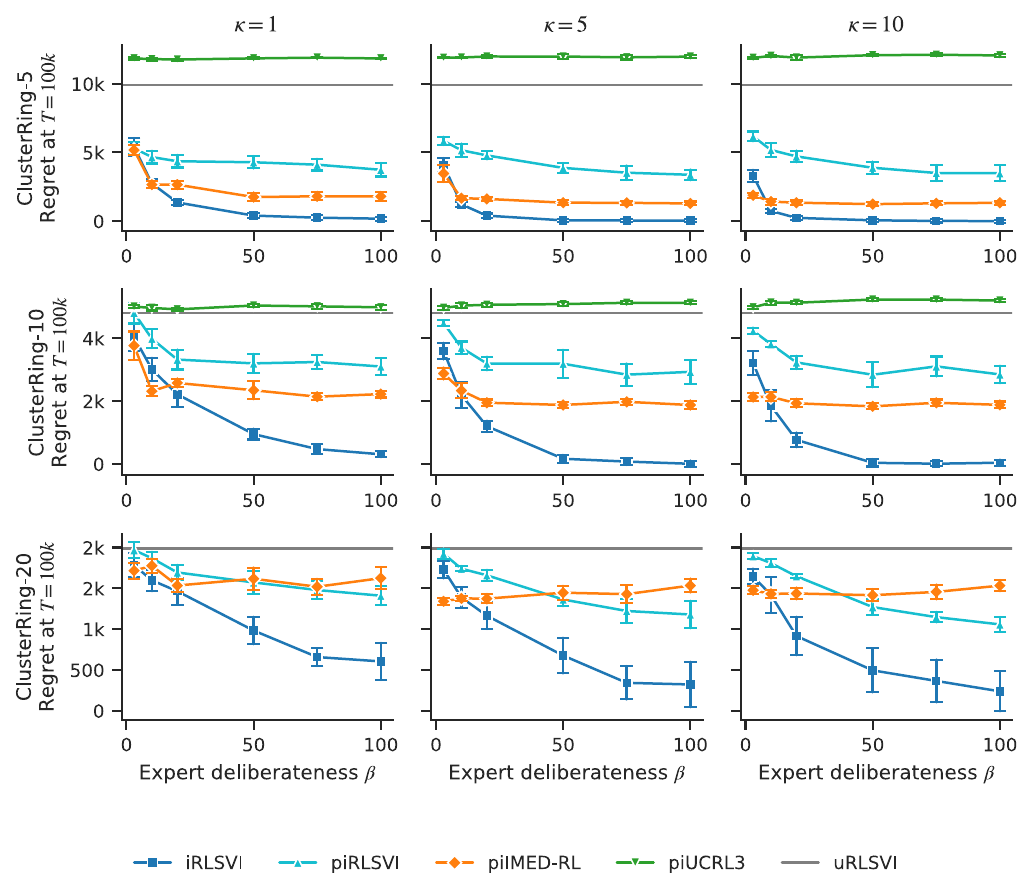}
\caption{Cumulative regret at $T=100{,}000$ online steps versus expert deliberateness $\beta$ for the \textit{ClusterRing} family.
Rows: environments; columns: data ratio $\kappa \in \{1, 5, 10\}$. Results are averaged over 10 seeds. Error bars show 95\% confidence intervals.}
\label{fig:app_regret_vs_beta_clusterring}
\end{figure}

\begin{figure}[t]
\centering
\includegraphics[width=\linewidth]{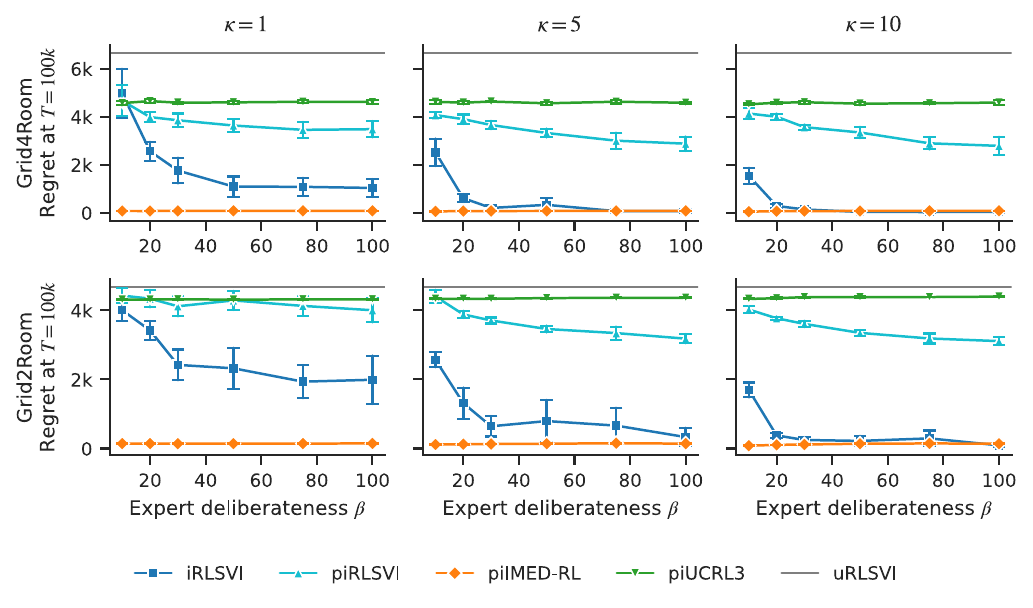}
\caption{Cumulative regret at $T=100{,}000$ online steps versus expert deliberateness $\beta$ for the Gridworld family.
Rows: environments; columns: data ratio $\kappa \in \{1, 5, 10\}$. Results are averaged over 10 seeds. Error bars show 95\% confidence intervals.}
\label{fig:app_regret_vs_beta_grid}
\end{figure}

\subsection{Per-environment regret curves}\label{app:percurves}

Figures~\ref{fig:per_env_riverswim6}--\ref{fig:per_env_grid2room} report per-environment regret curves on log--log axes for all $(\beta,\kappa)$ configurations. On the hard-exploration \textit{RiverSwim} and \textit{ClusterRing} benchmarks, informative demonstrations generally help iRLSVI achieve lower cumulative regret than online-only baselines, with gains depending on $\beta$ and $\kappa$.

\begin{figure}[t]
\centering
\includegraphics[width=\linewidth]{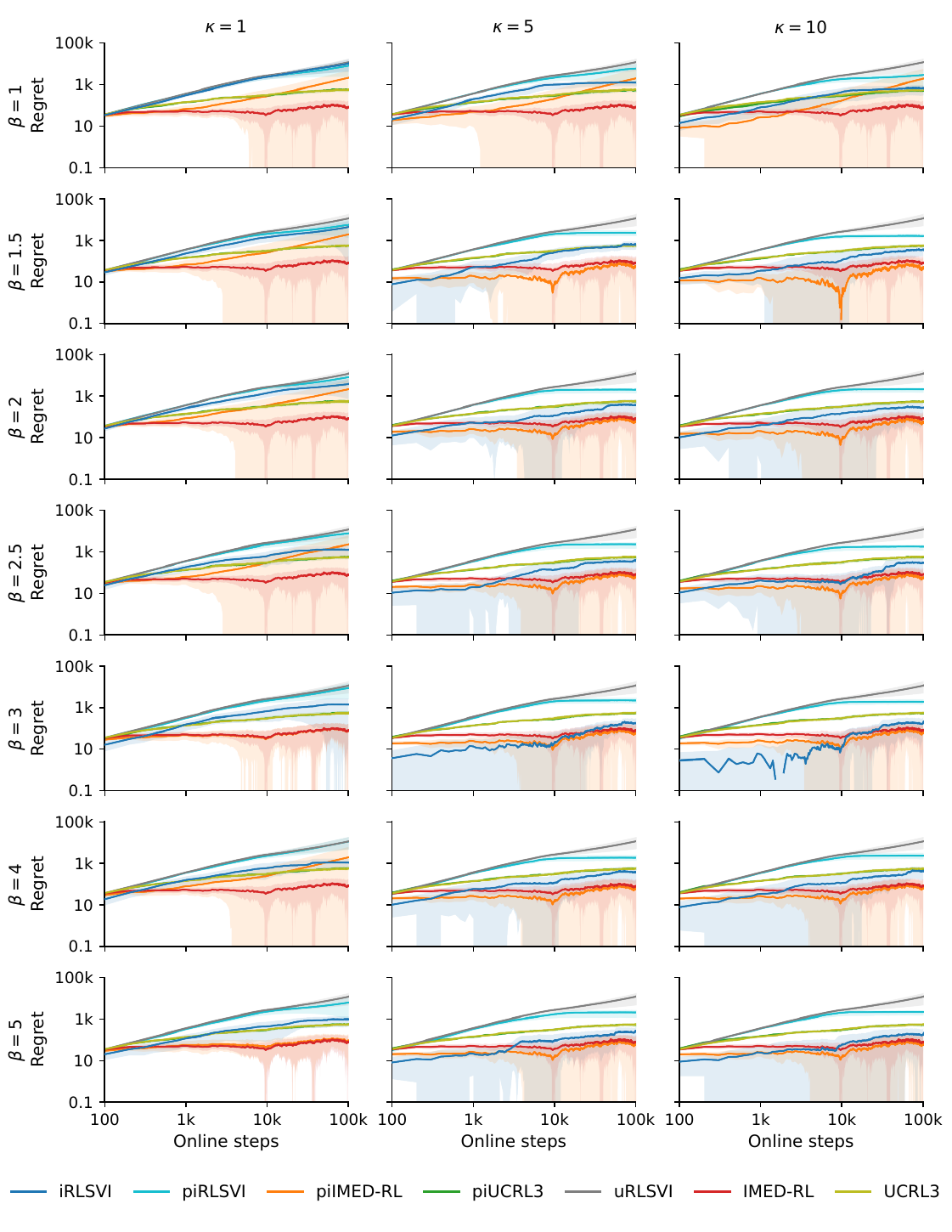}
\caption{Cumulative regret (log-log scale) for \textit{RiverSwim}-6.
Rows: expert deliberateness $\beta$; columns: data ratio $\kappa$. Results are averaged over 20 seeds. Shaded bands show two-sided 95\% Student-$t$ confidence intervals across seeds.}
\label{fig:per_env_riverswim6}
\end{figure}

\begin{figure}[t]
\centering
\includegraphics[width=\linewidth]{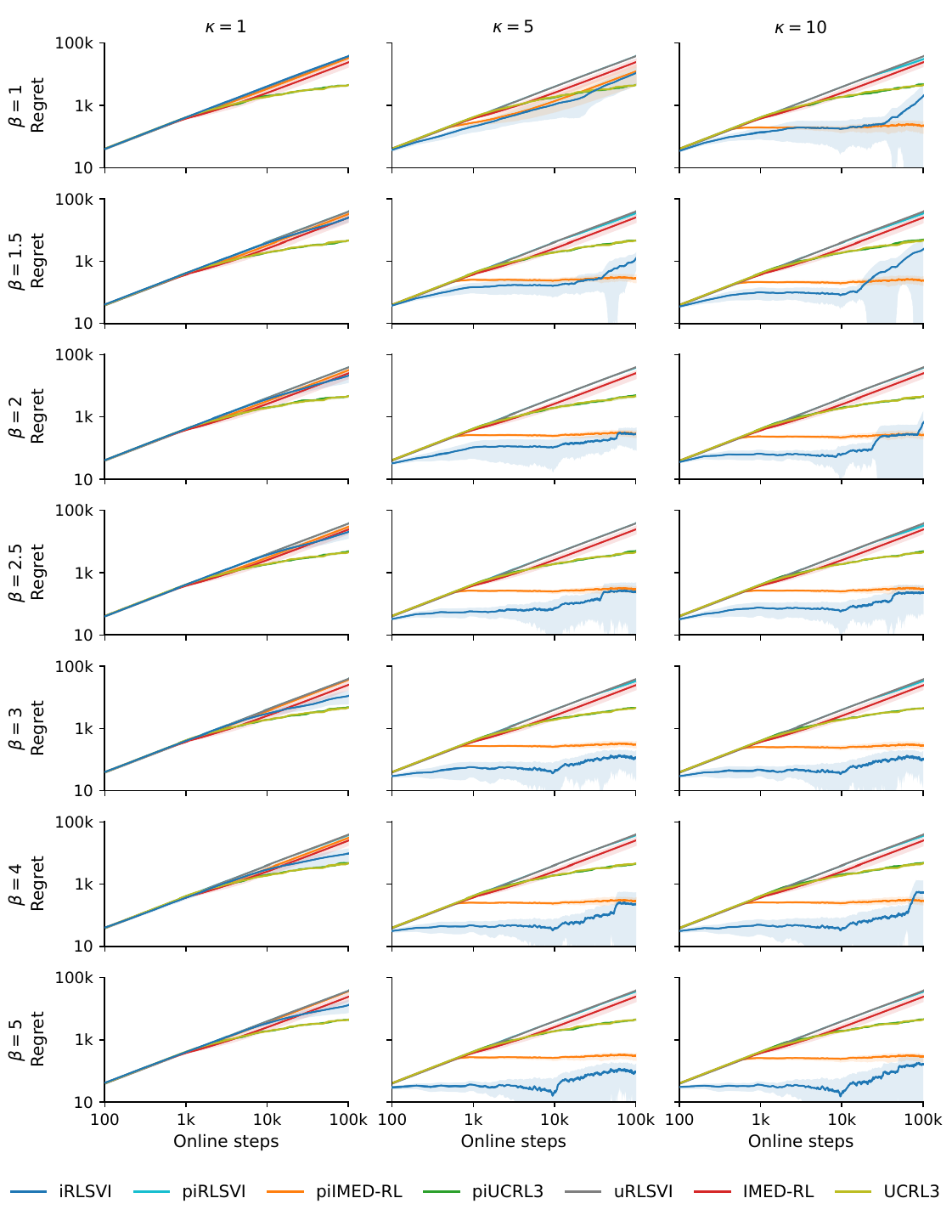}
\caption{Cumulative regret (log-log scale) for \textit{RiverSwim}-25.
Rows: expert deliberateness $\beta$; columns: data ratio $\kappa$. Results are averaged over 20 seeds. Shaded bands show two-sided 95\% Student-$t$ confidence intervals across seeds.}
\label{fig:per_env_riverswim25}
\end{figure}

\begin{figure}[t]
\centering
\includegraphics[width=\linewidth]{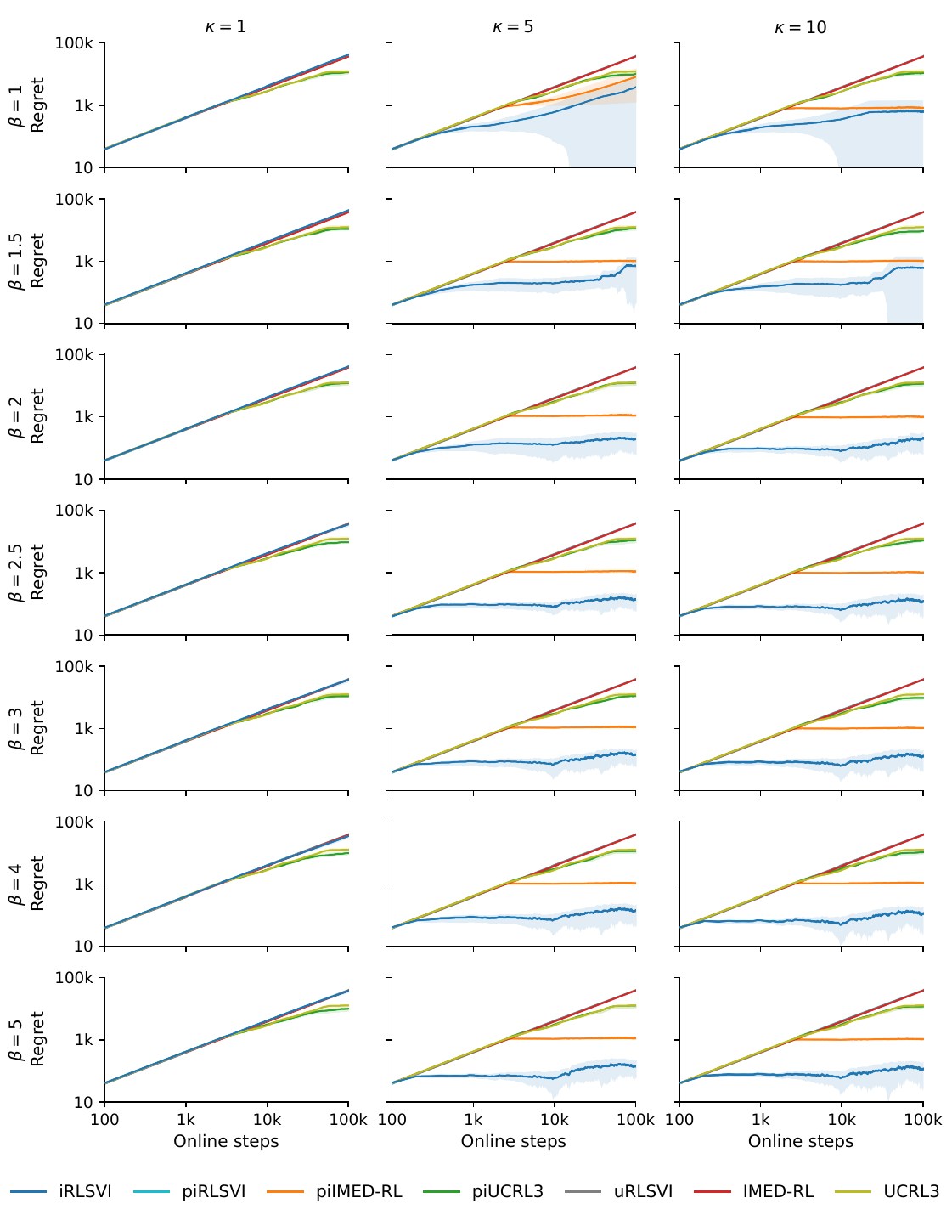}
\caption{Cumulative regret (log-log scale) for \textit{RiverSwim}-50.
Rows: expert deliberateness $\beta$; columns: data ratio $\kappa$. Results are averaged over 20 seeds. Shaded bands show two-sided 95\% Student-$t$ confidence intervals across seeds.}
\label{fig:per_env_riverswim50}
\end{figure}

\begin{figure}[t]
\centering
\includegraphics[width=\linewidth]{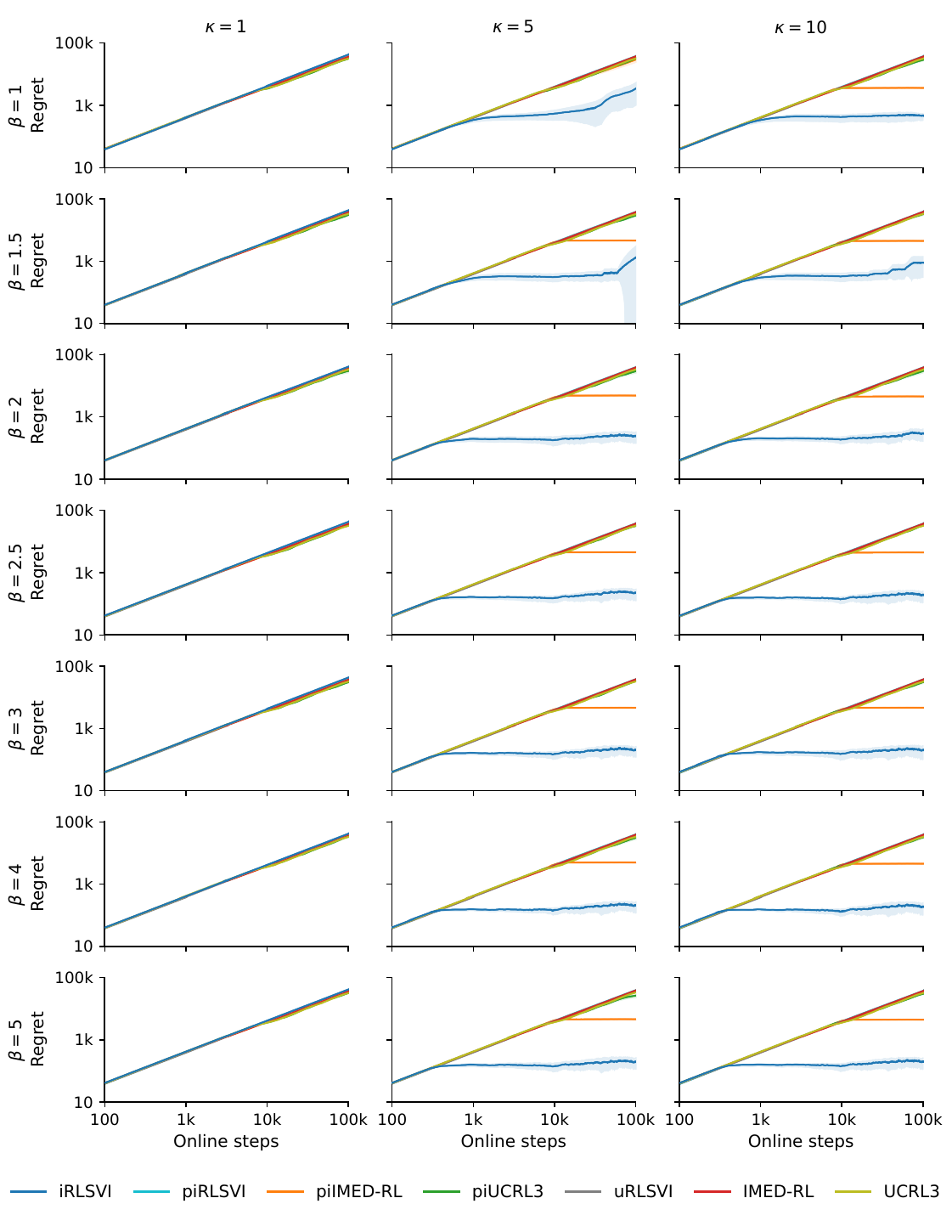}
\caption{Cumulative regret (log-log scale) for \textit{RiverSwim}-100.
Rows: expert deliberateness $\beta$; columns: data ratio $\kappa$. Results are averaged over 20 seeds. Shaded bands show two-sided 95\% Student-$t$ confidence intervals across seeds.}
\label{fig:per_env_riverswim100}
\end{figure}

\begin{figure}[t]
\centering
\includegraphics[width=\linewidth]{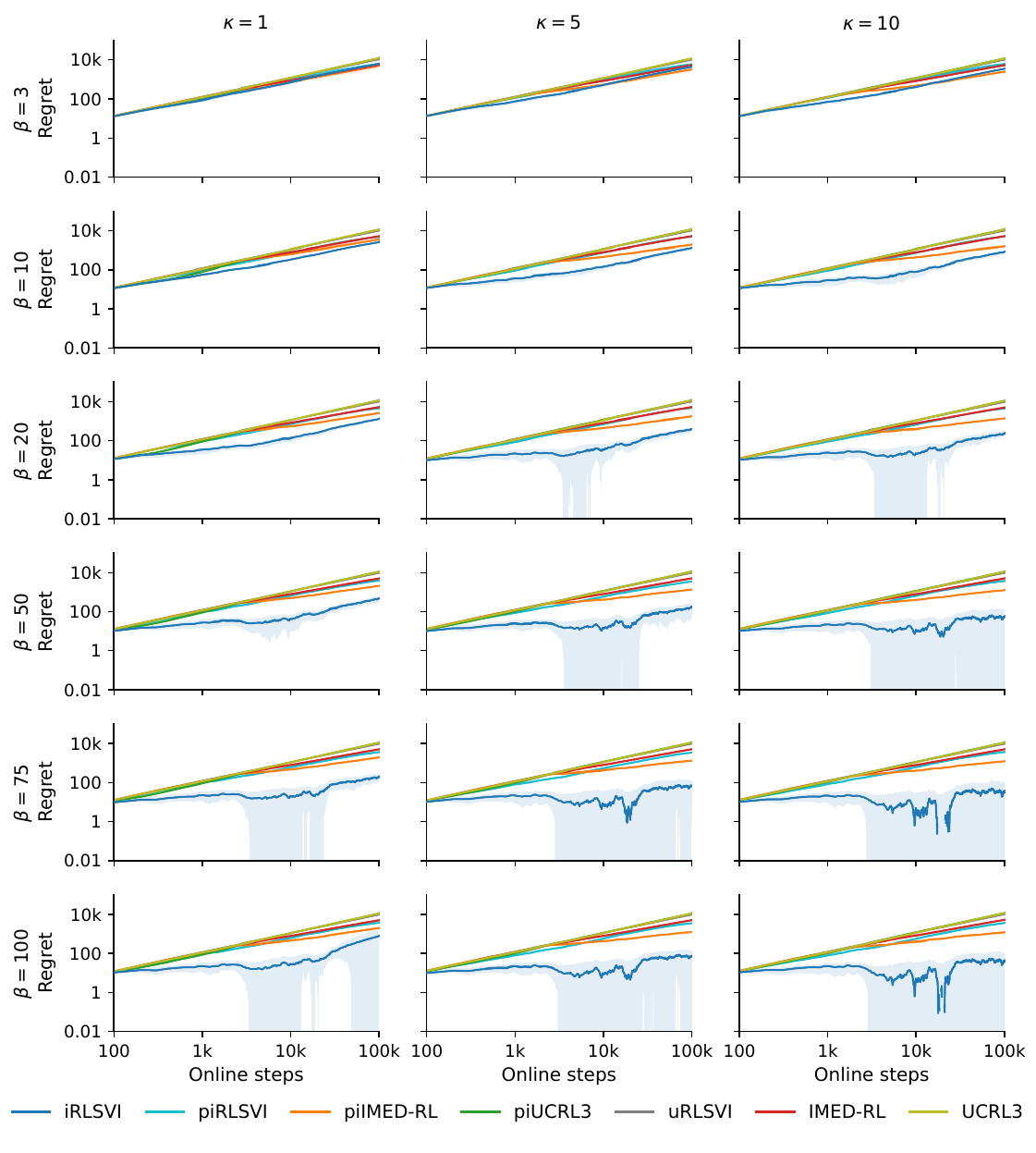}
\caption{Cumulative regret (log-log scale) for \textit{ClusterRing}-5.
Rows: expert deliberateness $\beta$; columns: data ratio $\kappa$. Results are averaged over 20 seeds. Shaded bands show two-sided 95\% Student-$t$ confidence intervals across seeds.}
\label{fig:per_env_clusterring5}
\end{figure}

\begin{figure}[t]
\centering
\includegraphics[width=\linewidth]{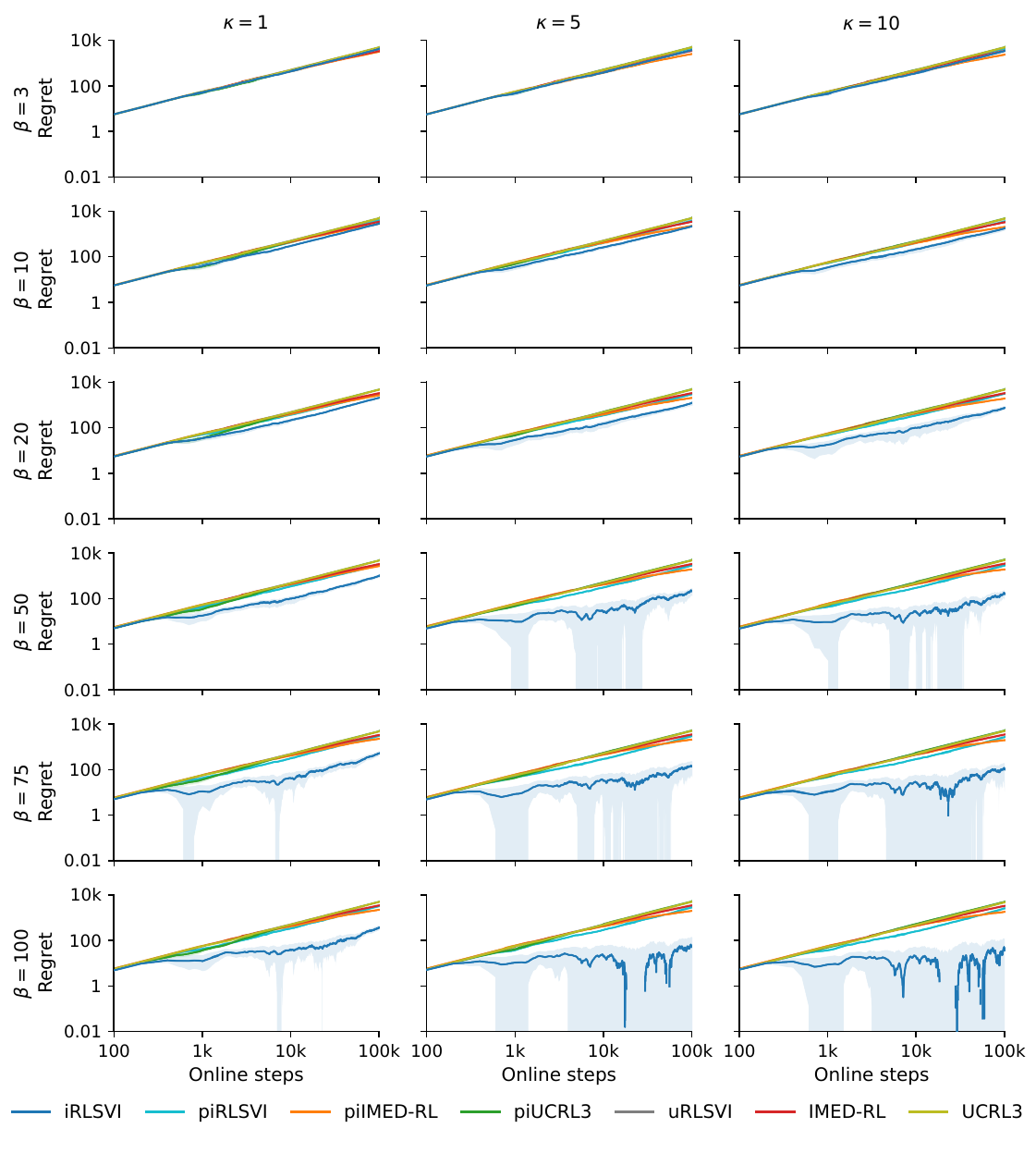}
\caption{Cumulative regret (log-log scale) for \textit{ClusterRing}-10.
Rows: expert deliberateness $\beta$; columns: data ratio $\kappa$. Results are averaged over 20 seeds. Shaded bands show two-sided 95\% Student-$t$ confidence intervals across seeds.}
\label{fig:per_env_clusterring10}
\end{figure}

\begin{figure}[t]
\centering
\includegraphics[width=\linewidth]{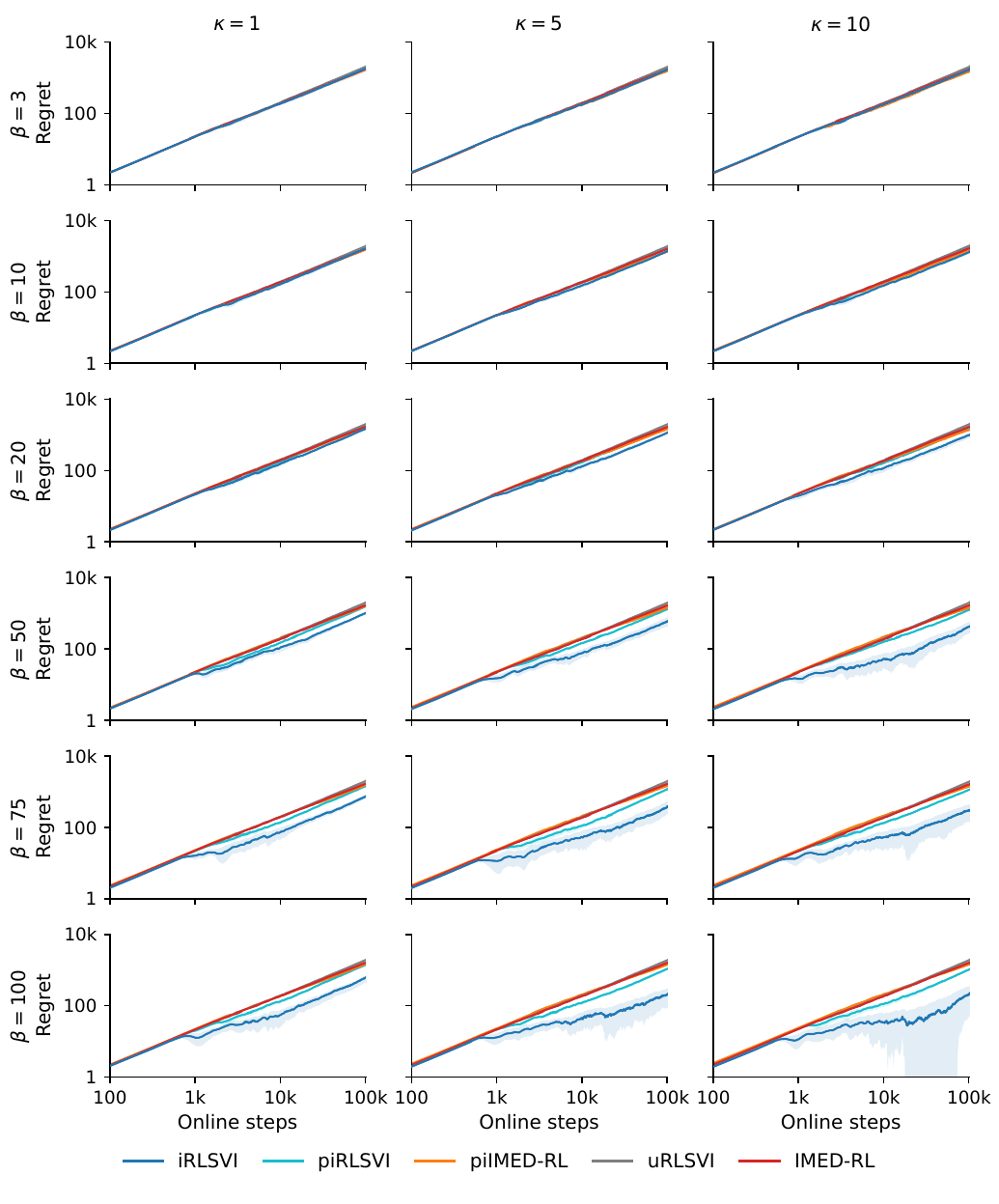}
\caption{Cumulative regret (log-log scale) for \textit{ClusterRing}-20.
Rows: expert deliberateness $\beta$; columns: data ratio $\kappa$. Results are averaged over 20 seeds. Shaded bands show two-sided 95\% Student-$t$ confidence intervals across seeds.}
\label{fig:per_env_clusterring20}
\end{figure}

\begin{figure}[t]
\centering
\includegraphics[width=\linewidth]{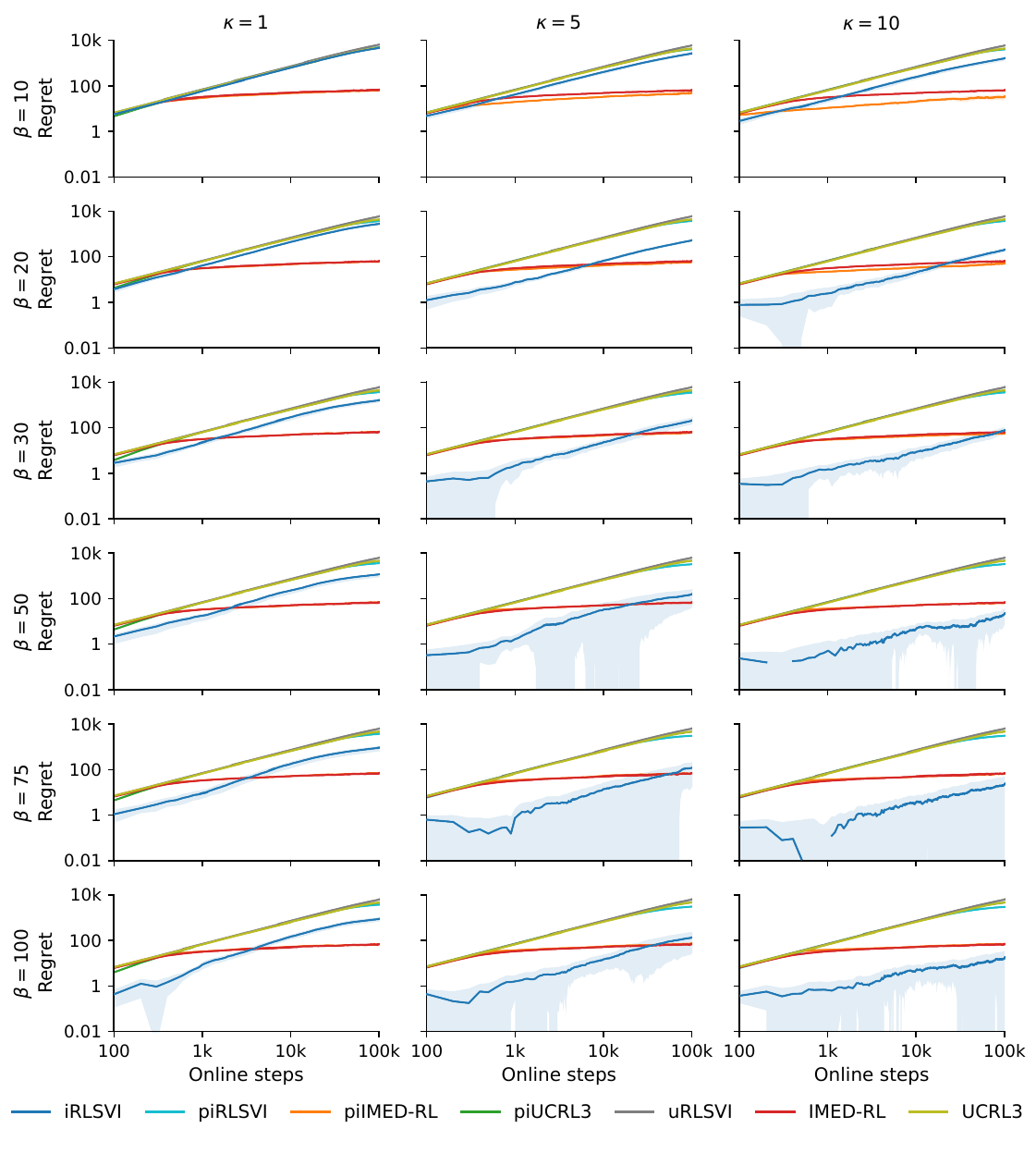}
\caption{Cumulative regret (log-log scale) for \textit{Grid4Room}.
Rows: expert deliberateness $\beta$; columns: data ratio $\kappa$. Results are averaged over 20 seeds. Shaded bands show two-sided 95\% Student-$t$ confidence intervals across seeds.}
\label{fig:per_env_grid4room}
\end{figure}

\begin{figure}[t]
\centering
\includegraphics[width=\linewidth]{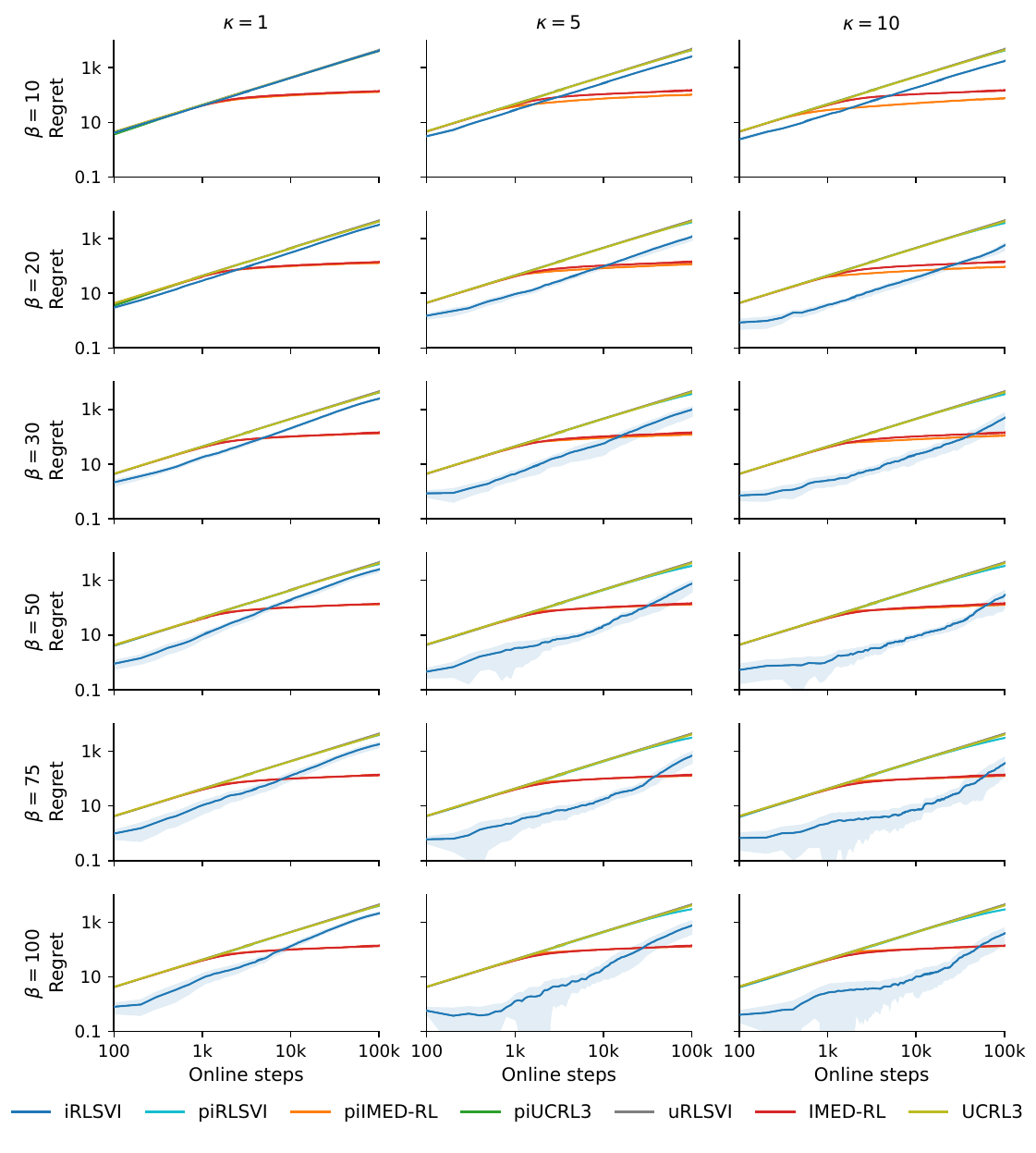}
\caption{Cumulative regret (log-log scale) for \textit{Grid2Room}.
Rows: expert deliberateness $\beta$; columns: data ratio $\kappa$. Results are averaged over 20 seeds. Shaded bands show two-sided 95\% Student-$t$ confidence intervals across seeds.}
\label{fig:per_env_grid2room}
\end{figure}

\clearpage
\ydzedit{
\subsection[Expert policy misspecification]{Expert policy misspecification}\label{app:robustness}
The regret analysis assumes the specified expert model. To test whether practical iRLSVI remains effective when this assumption is violated, we consider biased values, state-dependent $\beta(s)$, non-Boltzmann behavior, and heterogeneous experts. For the biased experts, $\tilde q(s,a)\sim\mathcal{N}(q^\star(s,a),\sigma_{\mathrm{env}}^2)$, with $\sigma_{\mathrm{env}}=0.2,0.02,0.05$ for \textit{RiverSwim}-50, \textit{ClusterRing}-10, and \textit{Grid4Room}, respectively, and a separate noise draw for each expert.
\begin{itemize}
\item \textbf{Boltzmann policy with $q^\star$, $\beta(s)$:} $\mathrm{softmax}(\beta(s)q^\star)$ with $\beta(s)\sim\mathrm{Exp}(1/\beta_{\mathrm{env}})$ drawn independently for each state and having mean $\beta_{\mathrm{env}}$. iRLSVI still fits one scalar $\beta$, isolating the effect of varying $\beta(s)$ from value bias.
\item \textbf{Boltzmann policy with $\tilde q$, scalar $\beta$:} $\mathrm{softmax}(\beta_{\mathrm{env}}\tilde q)$ at the per-environment $\beta_{\mathrm{env}}$ of Table~\ref{tab:main_results} (3 / 50 / 50), isolating value bias.
\item \textbf{Boltzmann policy with $\tilde q$, $\beta(s)$:} $\mathrm{softmax}(\beta(s)\tilde q)$ with $\beta(s)\sim\mathrm{Exp}(1/\beta_{\mathrm{env}})$ drawn independently for each state, with mean $\beta_{\mathrm{env}}$. iRLSVI still fits one scalar $\beta$.
\item \textbf{$\epsilon$-greedy policy with $\tilde q$:} a non-Boltzmann rule that acts randomly with probability $\epsilon=0.2$ and otherwise chooses $\arg\max_a\tilde q(s,a)$.
\item \textbf{Mixture of misspecified experts:} demonstrations from the three biased experts above that act on $\tilde q$ rather than exact $q^\star$: the scalar-$\beta$ Boltzmann, $\beta(s)$ Boltzmann, and $\epsilon$-greedy experts.
\end{itemize}

\phantomsection\label{app:robustness-mixture}
We fix the data ratio at $\kappa=5$, giving offline budgets of $N=5SA$ transitions: 500 for \textit{RiverSwim}-50, 5,000 for \textit{ClusterRing}-10, and 400 for \textit{Grid4Room}. The balanced mixture takes the first $\lfloor N/3\rfloor$ transitions from each of the three source trajectories. The pooled mixture samples $N$ transitions uniformly without replacement from their union. Within each setting, agents receive the same offline data for a given seed. Seeds vary the offline trajectories and online interaction while holding the expert policies fixed. For the misspecification experiments, both RLSVI variants update their policies every 100 steps and their target parameters every 1,000 steps. iRLSVI uses an exponential prior of rate 1 for $\beta$, initializes $\beta$ to 1, and initializes its reward-rate estimate to the offline mean reward.

\textbf{Finding.} Across the four individual policies and the balanced mixture in all three environments, iRLSVI achieves lower mean regret than uRLSVI and the direct piRLSVI ablation in all 15 comparisons. It also achieves lower mean regret than piIMED-RL in 9 of these 15 comparisons. Under pooled sampling on \textit{RiverSwim}-50, iRLSVI and piRLSVI both incur about 38,500 mean regret, compared with 1,171 for piIMED-RL.

\textbf{Cross-row comparisons.} Comparisons within each row use the same expert. Across expert types, both expert quality and state-action occupancy can change, so cross-row regret differences cannot be attributed to policy-form misspecification alone.

\clearpage
Tables~\ref{tab:expert-misspecification-rs50}--\ref{tab:expert-misspecification-g4} give results at $T=100{,}000$ and $\kappa=5$. We report mean cumulative regret over 20 seeds with 95\% confidence intervals. Quality is the expert's empirical average reward over the full source trajectories as a percentage of $\bar r^\star$. Both mixture rows report the equal-weight average quality of the same three source datasets. The well-specified expert from Table~\ref{tab:main_results} is included for context. Bold indicates the lowest mean regret among the three methods in each row.

\begin{table}[!ht]\ydzedit{\centering
\caption{\ydzedit{\textit{RiverSwim-50}. Mean regret for online-only uRLSVI (Table~\ref{tab:main_results}) is 38,558 $\pm$ 184.}}\label{tab:expert-misspecification-rs50}
\setlength{\tabcolsep}{3pt}
\ydzedit{\begin{tabular}{lrrrr}
\toprule
Expert & Quality & iRLSVI & piRLSVI & piIMED-RL \\
\midrule
Boltzmann, $q^\star$, scalar $\beta$ (Table~\ref{tab:main_results}) & 93.32\% & \textbf{137 $\pm$ 79} & 38,726 $\pm$ 240 & 1,134 $\pm$ 107 \\
Boltzmann, $q^\star$, $\beta(s)$ & 83.25\% & \textbf{5,978 $\pm$ 4,351} & 38,544 $\pm$ 193 & 8,342 $\pm$ 7,086 \\
Boltzmann, $\tilde q$, $\beta$ & 71.65\% & 2,338 $\pm$ 1,287 & 38,745 $\pm$ 245 & \textbf{1,114 $\pm$ 108} \\
Boltzmann, $\tilde q$, $\beta(s)$ & 83.34\% & \textbf{8,993 $\pm$ 4,787} & 38,605 $\pm$ 210 & 13,859 $\pm$ 8,453 \\
$\epsilon$-greedy, $\tilde q$ & 62.71\% & \textbf{1,793 $\pm$ 1,775} & 38,658 $\pm$ 237 & 2,878 $\pm$ 3,864 \\
Balanced mixture & 72.57\% & \textbf{5,051 $\pm$ 2,975} & 39,335 $\pm$ 623 & 37,837 $\pm$ 8 \\
Pooled mixture & 72.57\% & 38,532 $\pm$ 183 & 38,538 $\pm$ 195 & \textbf{1,171 $\pm$ 131} \\
\bottomrule
\end{tabular}}
}\end{table}

\begin{table}[!ht]\ydzedit{\centering
\caption{\ydzedit{\textit{ClusterRing-10}. Mean regret for online-only uRLSVI (Table~\ref{tab:main_results}) is 4,975 $\pm$ 245.}}\label{tab:expert-misspecification-cr10}
\setlength{\tabcolsep}{3pt}
\ydzedit{\begin{tabular}{lrrrr}
\toprule
Expert & Quality & iRLSVI & piRLSVI & piIMED-RL \\
\midrule
Boltzmann, $q^\star$, scalar $\beta$ (Table~\ref{tab:main_results}) & 82.70\% & \textbf{213 $\pm$ 78} & 2,986 $\pm$ 226 & 2,015 $\pm$ 120 \\
Boltzmann, $q^\star$, $\beta(s)$ & 71.19\% & \textbf{925 $\pm$ 142} & 3,278 $\pm$ 212 & 2,060 $\pm$ 140 \\
Boltzmann, $\tilde q$, $\beta$ & 80.81\% & \textbf{588 $\pm$ 83} & 2,871 $\pm$ 243 & 2,073 $\pm$ 115 \\
Boltzmann, $\tilde q$, $\beta(s)$ & 70.97\% & \textbf{1,219 $\pm$ 115} & 3,122 $\pm$ 192 & 2,060 $\pm$ 108 \\
$\epsilon$-greedy, $\tilde q$ & 69.80\% & \textbf{787 $\pm$ 91} & 3,881 $\pm$ 167 & 2,078 $\pm$ 89 \\
Balanced mixture & 73.86\% & \textbf{638 $\pm$ 94} & 3,517 $\pm$ 212 & 2,025 $\pm$ 93 \\
Pooled mixture & 73.86\% & \textbf{619 $\pm$ 90} & 3,224 $\pm$ 164 & 2,026 $\pm$ 102 \\
\bottomrule
\end{tabular}}
}\end{table}

\begin{table}[!ht]\ydzedit{\centering
\caption{\ydzedit{\textit{Grid4Room}. Mean regret for online-only uRLSVI (Table~\ref{tab:main_results}) is 6,252 $\pm$ 652.}}\label{tab:expert-misspecification-g4}
\setlength{\tabcolsep}{3pt}
\ydzedit{\begin{tabular}{lrrrr}
\toprule
Expert & Quality & iRLSVI & piRLSVI & piIMED-RL \\
\midrule
Boltzmann, $q^\star$, scalar $\beta$ (Table~\ref{tab:main_results}) & 96.84\% & 161 $\pm$ 118 & 3,228 $\pm$ 166 & \textbf{70 $\pm$ 10} \\
Boltzmann, $q^\star$, $\beta(s)$ & 78.33\% & 655 $\pm$ 142 & 3,490 $\pm$ 158 & \textbf{66 $\pm$ 9} \\
Boltzmann, $\tilde q$, $\beta$ & 81.27\% & 1,713 $\pm$ 450 & 3,476 $\pm$ 190 & \textbf{73 $\pm$ 10} \\
Boltzmann, $\tilde q$, $\beta(s)$ & 46.19\% & 2,946 $\pm$ 328 & 3,782 $\pm$ 139 & \textbf{199 $\pm$ 44} \\
$\epsilon$-greedy, $\tilde q$ & 65.97\% & 1,480 $\pm$ 226 & 3,880 $\pm$ 142 & \textbf{137 $\pm$ 39} \\
Balanced mixture & 64.48\% & 1,252 $\pm$ 299 & 3,733 $\pm$ 118 & \textbf{89 $\pm$ 29} \\
Pooled mixture & 64.48\% & 1,299 $\pm$ 170 & 3,725 $\pm$ 140 & \textbf{82 $\pm$ 28} \\
\bottomrule
\end{tabular}}
}\end{table}

\textbf{Statistical uncertainty.} On \textit{RiverSwim}-50, the intervals for iRLSVI remain wide under state-dependent $\beta(s)$ and the balanced mixture (Table~\ref{tab:expert-misspecification-rs50}).

\clearpage
\subsection[Comparison with offline-to-online baselines]{Comparison with offline-to-online baselines}\label{app:o2o}
Table~\ref{tab:prior-main} compares iRLSVI with its partially informed ablation piRLSVI and two adaptations of prior methods. All four methods receive the same offline transitions for each seed.

\begin{table}[!ht]\ydzedit{\centering
\caption{\ydzedit{Cumulative regret at $T=100{,}000$ online steps for the offline-to-online baselines and iRLSVI ($\kappa=5$). We report mean cumulative regret over 20 seeds with 95\% confidence intervals. CQL-FT uses discounted conservative Q-learning \citep{kumar2020conservative} and resets every 1,000 steps. For CQL-FT, we report the lowest mean cumulative regret across the four discount factors tested in each environment.}}\label{tab:prior-main}
\setlength{\tabcolsep}{3pt}\ydzedit{\begin{tabular}{lrrrr}
\toprule
Environment & iRLSVI & piRLSVI & RLPD-tab & CQL-FT \\
\midrule
RiverSwim-25 & 108 $\pm$ 87 & 33,455 $\pm$ 5,067 & 3,233 $\pm$ 731 & 7,157 $\pm$ 3,676 \\
RiverSwim-50 & 137 $\pm$ 79 & 38,726 $\pm$ 240 & 3,553 $\pm$ 791 & 10,921 $\pm$ 4,326 \\
RiverSwim-100 & 202 $\pm$ 82 & 38,631 $\pm$ 202 & 10,416 $\pm$ 3,679 & 17,728 $\pm$ 2,673 \\
ClusterRing-10 & 213 $\pm$ 78 & 2,986 $\pm$ 226 & 1,771 $\pm$ 202 & 1,318 $\pm$ 105 \\
Grid4Room & 161 $\pm$ 118 & 3,228 $\pm$ 166 & 4,713 $\pm$ 189 & 266 $\pm$ 145 \\
\bottomrule\end{tabular}}}\end{table}

\textbf{RLPD-tab.} We adapt the symmetric offline and online sampling of RLPD \citep{ball2023efficient} to tabular average-reward learning. In each full-batch update with $t>0$ online transitions, each online transition has weight 1 and each of the $N$ offline transitions has weight $t/N$. The two datasets thus contribute equal total weight. Before the first online transition, each offline transition has weight 1. We update the policy and refresh its target every 100 online steps.

\textbf{CQL-FT.} We pretrain discrete conservative Q-learning \citep{kumar2020conservative} for 50,000 gradient steps, then fine-tune online with the offline data in the replay buffer. The Q-network takes one-hot state encodings and has two hidden layers of 64 units each. We use a batch size of 64, a learning rate of $10^{-3}$, a conservative penalty weight of 1, and a target update interval of 1,000 gradient steps. The environment resets every 1,000 online steps.

For CQL-FT, Table~\ref{tab:prior-main} reports the lowest mean cumulative regret over $\gamma\in\{0.95,0.99,0.999,0.9999\}$ in each environment. The selected discount is 0.9999 for \textit{RiverSwim}-25 and -50 and 0.95 for \textit{RiverSwim}-100, \textit{ClusterRing}-10, and \textit{Grid4Room}. Each interval is computed across seeds at the selected discount and does not account for discount selection.

\subsection{Effect of offline demonstration quantity}\label{app:data-size}
We vary the offline data ratio $\kappa=N/(SA)\in\{1,2,3,5,10\}$ using the expert deliberateness from Table~\ref{tab:main_results}. We refresh the target every 100 steps and use an exponential prior of rate 1 for $\beta$.
\begin{table}[ht]\ydzedit{\centering
\caption{\ydzedit{Cumulative regret at $T=100{,}000$ online steps for iRLSVI as the data ratio $\kappa$ varies. We report mean cumulative regret over 20 seeds with 95\% confidence intervals.}}
\label{tab:offline-demonstration-quantity}
\setlength{\tabcolsep}{3pt}\ydzedit{\begin{tabular}{lrrr}
\toprule
$\kappa$ & RiverSwim-50 & ClusterRing-10 & Grid4Room \\
\midrule
1 & 42,813 $\pm$ 1 & 919 $\pm$ 152 & 1,029 $\pm$ 157 \\
2 & 4,484 $\pm$ 6,136 & 619 $\pm$ 89 & 621 $\pm$ 101 \\
3 & 206 $\pm$ 95 & 478 $\pm$ 88 & 442 $\pm$ 100 \\
5 & 248 $\pm$ 97 & 325 $\pm$ 73 & 202 $\pm$ 69 \\
10 & 257 $\pm$ 152 & 171 $\pm$ 81 & 51 $\pm$ 27 \\
\bottomrule\end{tabular}}}\end{table}

\subsection{Effect of episode schedules}\label{app:episode-schedule}
The growing schedule $T_k=k$ yields the regret bound for iPSRL, but need not minimize empirical regret for its practical approximation, iRLSVI. We compare fixed episode lengths $T_k\in\{1,10,100,1000\}$ with the growing schedule $T_k=k$, using $\kappa=5$ and the expert deliberateness from Table~\ref{tab:main_results}. For each schedule, we refresh the target at the end of each episode, initialize the reward-rate estimate to zero, and use an exponential prior of rate 1 for $\beta$.
\begin{table}[ht]\ydzedit{\centering
\caption{\ydzedit{Cumulative regret at $T=100{,}000$ online steps for iRLSVI under different episode schedules. We report mean cumulative regret over 20 seeds with 95\% confidence intervals.}}
\label{tab:episode-schedules}
\setlength{\tabcolsep}{3pt}\ydzedit{\begin{tabular}{lrrr}
\toprule
$T_k$ & RiverSwim-50 & ClusterRing-10 & Grid4Room \\
\midrule
1 & 14,563 $\pm$ 9,148 & 234 $\pm$ 83 & 24 $\pm$ 11 \\
10 & 209 $\pm$ 94 & 257 $\pm$ 71 & 100 $\pm$ 42 \\
100 & 220 $\pm$ 84 & 288 $\pm$ 79 & 185 $\pm$ 60 \\
1000 & 1,455 $\pm$ 755 & 283 $\pm$ 81 & 301 $\pm$ 204 \\
$k$ & 136 $\pm$ 80 & 312 $\pm$ 77 & 265 $\pm$ 91 \\
\bottomrule\end{tabular}}}\end{table}

}

%%%%%%%%%%%%%%%%%%%%%%%%%%%%%%%%%%%%%%%%%%%%%%%%%%%%%%%%%%%%

\clearpage
\section*{NeurIPS Paper Checklist}

\begin{enumerate}

\item {\bf Claims}
    \item[] Question: Do the main claims made in the abstract and introduction accurately reflect the paper's contributions and scope?
    \item[] Answer: \answerYes{} % Replace by \answerYes{}, \answerNo{}, or \answerNA{}.
    \item[] Justification: The abstract and introduction accurately reflect the paper's contributions or scope. All main claims are supported by theoretical results in Section \ref{sec:priordependentbound} (and proofs in appendices) and/or experimental results in Section \ref{sec:experiments} (and more experimental results in appendices).

\item {\bf Limitations}
    \item[] Question: Does the paper discuss the limitations of the work performed by the authors?
    \item[] Answer: \answerYes{} % Replace by \answerYes{}, \answerNo{}, or \answerNA{}.
    \item[] Justification: We discuss limitations in Section \ref{sec:conclusions}. Additionally, although we have covered several synthetic tabular MDPs for the experiment, we acknowledge that the experiments could be further improved by incorporating (even) larger and more realistic environments.  

\item {\bf Theory assumptions and proofs}
    \item[] Question: For each theoretical result, does the paper provide the full set of assumptions and a complete (and correct) proof?
    \item[] Answer: \answerYes{} % Replace by \answerYes{}, \answerNo{}, or \answerNA{}.
    \item[] Justification: We state and label all assumptions clearly in Sections~\ref{sec:prelims} and~\ref{sec:bounding}. Complete proofs for theoretical results (or derivations) are also provided in Appendices~\ref{app:proof-regbound}, \ref{app:proof-estimator2p1}, and \ref{app:proof-piestimator}.

    \item {\bf Experimental result reproducibility}
    \item[] Question: Does the paper fully disclose all the information needed to reproduce the main experimental results of the paper to the extent that it affects the main claims and/or conclusions of the paper (regardless of whether the code and data are provided or not)?
    \item[] Answer: \answerYes{} % Replace by \answerYes{}, \answerNo{}, or \answerNA{}.
    \item[] Justification: We report all relevant information to reproduce experimental results of the paper in Appendix~\ref{app:experiments}.

\item {\bf Open access to data and code}
    \item[] Question: Does the paper provide open access to the data and code, with sufficient instructions to faithfully reproduce the main experimental results, as described in supplemental material?
    \item[] Answer: \answerYes{} % Replace by \answerYes{}, \answerNo{}, or \answerNA{}.
    \item[] Justification: Code and experiment configurations are available in the public repository linked in Section~\ref{sec:experiments}, together with instructions for generating offline demonstrations and running the methods.

\item {\bf Experimental setting/details}
    \item[] Question: Does the paper specify all the training and test details (e.g., data splits, hyperparameters, how they were chosen, type of optimizer) necessary to understand the results?
    \item[] Answer: \answerYes{} % Replace by \answerYes{}, \answerNo{}, or \answerNA{}.
    \item[] Justification: We report all relevant information to reproduce and understand the experimental results of the paper in Appendix~\ref{app:experiments}.

\item {\bf Experiment statistical significance}
    \item[] Question: Does the paper report error bars suitably and correctly defined or other appropriate information about the statistical significance of the experiments?
    \item[] Answer: \answerYes{} % Replace by \answerYes{}, \answerNo{}, or \answerNA{}.
    \item[] Justification: We report 95\% confidence intervals in Figure~\ref{fig:regret_vs_beta_main} and Appendices~\ref{app:beta_grid}--\ref{app:episode-schedule}. Tables~\ref{tab:main_results} and~\ref{tab:robustness-main} report mean cumulative regret without intervals; intervals for the robustness results are provided in Appendix~\ref{app:robustness}. Appendix~\ref{app:percurves} provides 20-seed learning curves. Seed counts and the construction of confidence intervals are specified in Appendix~\ref{app:variance}.

\item {\bf Experiments compute resources}
    \item[] Question: For each experiment, does the paper provide sufficient information on the computer resources (type of compute workers, memory, time of execution) needed to reproduce the experiments?
    \item[] Answer: \answerYes{} % Replace by \answerYes{}, \answerNo{}, or \answerNA{}.
    \item[] Justification: We describe the compute resource needed for our experiments in Appendix~\ref{app:compute}. Overall, our experiments used minimal compute.
    
\item {\bf Code of ethics}
    \item[] Question: Does the research conducted in the paper conform, in every respect, with the NeurIPS Code of Ethics \url{https://neurips.cc/public/EthicsGuidelines}?
    \item[] Answer: \answerYes{} % Replace by \answerYes{}, \answerNo{}, or \answerNA{}.
    \item[] Justification: The research presented in this paper conforms, in every respect, with the NeurIPS Code of Ethics.

\item {\bf Broader impacts}
    \item[] Question: Does the paper discuss both potential positive societal impacts and negative societal impacts of the work performed?
    \item[] Answer: \answerNA{} % Replace by \answerYes{}, \answerNo{}, or \answerNA{}.
    \item[] Justification: We believe there is no negative societal impact that our work might lead to directly. This work is about basic research, while certain positive societal impact may arise from it (e.g., medical applications that require both offline and online reinforcement learning), none of which we believe should be discussed specifically.
    
\item {\bf Safeguards}
    \item[] Question: Does the paper describe safeguards that have been put in place for responsible release of data or models that have a high risk for misuse (e.g., pre-trained language models, image generators, or scraped datasets)?
    \item[] Answer: \answerNA{} % Replace by \answerYes{}, \answerNo{}, or \answerNA{}.
    \item[] Justification: The paper poses no such risks.

\item {\bf Licenses for existing assets}
    \item[] Question: Are the creators or original owners of assets (e.g., code, data, models), used in the paper, properly credited and are the license and terms of use explicitly mentioned and properly respected?
    \item[] Answer: \answerYes{} % Replace by \answerYes{}, \answerNo{}, or \answerNA{}.
    \item[] Justification: The creators or original owners of assets used in the paper are properly credited and license information fully disclosed (see footnotes of page 8).

\item {\bf New assets}
    \item[] Question: Are new assets introduced in the paper well documented and is the documentation provided alongside the assets?
    \item[] Answer: \answerYes{} % Replace by \answerYes{}, \answerNo{}, or \answerNA{}.
    \item[] Justification: The paper releases a few new assets (code), such as an implementation of the proposed algorithm along with a new synthetic RL environment for testing exploration capability of the algorithm. 

\item {\bf Crowdsourcing and research with human subjects}
    \item[] Question: For crowdsourcing experiments and research with human subjects, does the paper include the full text of instructions given to participants and screenshots, if applicable, as well as details about compensation (if any)? 
    \item[] Answer: \answerNA{} % Replace by \answerYes{}, \answerNo{}, or \answerNA{}.
    \item[] Justification: The paper does not involve crowdsourcing nor research with human subjects.

\item {\bf Institutional review board (IRB) approvals or equivalent for research with human subjects}
    \item[] Question: Does the paper describe potential risks incurred by study participants, whether such risks were disclosed to the subjects, and whether Institutional Review Board (IRB) approvals (or an equivalent approval/review based on the requirements of your country or institution) were obtained?
    \item[] Answer: \answerNA{} % Replace by \answerYes{}, \answerNo{}, or \answerNA{}.
    \item[] Justification: The paper does not involve crowdsourcing nor research with human subjects.

\item {\bf Declaration of LLM usage}
    \item[] Question: Does the paper describe the usage of LLMs if it is an important, original, or non-standard component of the core methods in this research? Note that if the LLM is used only for writing, editing, or formatting purposes and does \emph{not} impact the core methodology, scientific rigor, or originality of the research, declaration is not required.
    %this research? 
    \item[] Answer: \answerNA{} % Replace by \answerYes{}, \answerNo{}, or \answerNA{}.
    \item[] Justification: The core method development in this research does not involve LLMs as any important, original, or non-standard components.

\end{enumerate}

\end{document}